\documentclass[10pt,reqno]{amsart}

\usepackage{amsthm,amsmath,amsfonts,amssymb}
\usepackage{bm}
\usepackage[dvipdfmx]{graphicx}
\usepackage{xcolor}
\usepackage{fullpage}
\usepackage{verbatim}

\usepackage{caption}
\usepackage{float}

\usepackage{tikz}
\usetikzlibrary{calc, positioning}
\usetikzlibrary{decorations.markings}
\usetikzlibrary{quotes,angles}
\usetikzlibrary{arrows.meta}
\usetikzlibrary{intersections}
\usetikzlibrary{graphs}

\usepackage{appendix}
\usepackage[sort&compress, numbers]{natbib}

\usepackage{hyperref}					
\hypersetup{colorlinks,					
	linkcolor=blue,
	citecolor=red}

\usepackage{amsaddr}
\usepackage{mathrsfs}

\newtheorem{theorem}{Theorem}
\newtheorem{corollary}{Corollary}
\newtheorem{lemma}{Lemma}
\newtheorem{proposition}{Proposition}

\theoremstyle{definition}

\newtheorem{remark}{Remark}

\def\R{\mathbb{R}}
\def\E{\mathbb{E}}

\def\N{\mathbb{N}}

\def\AA{\mathcal{A}}
\def\BB{\mathcal{B}}
\def\MM{\mathcal{M}}

\def\EE{\mathcal{E}}
\def\FF{\mathcal{F}}
\def\LL{\mathcal{L}}
\def\RR{\mathcal{R}}
\def\OO{\mathcal{O}}
\def\HH{\mathcal{H}}

\def\NN{\mathcal{N}}

\def\GP{\mathcal{GP}}
\def\LD{\mathscr{D}}
\def\PLDR{\mathcal{M}} 

\newcommand{\ol}{\overline}
\newcommand{\wt}{\widetilde}

\newcommand{\Cov}{\operatorname{Cov}}
\renewcommand{\d}{\mathrm{d}}

\def\llan{\left\langle}
\def\rran{\right\rangle}

\allowdisplaybreaks

\title{The Symmetry of Differential Neural Tangent Kernel \\ under the Orthogonal Group}
\author{ Bangti Jin\and Longjun Wu} 
\address{ Department of Mathematics, The Chinese University of Hong Kong, Shatin, N.T., Hong Kong, P.R. China }

\begin{document}

\begin{abstract}
The differential neural tangent kernel (DNTK) is a tool for analyzing the training dynamics of over-parameterized physics-informed neural networks. In this work, we investigate the symmetry of the infinite DNTK under an orthogonal group, for two‑layer and deep neural networks separately, based on the explicit expression of the DNTK and recursive formula, respectively. The results are for all activations with or without a bias term, and shed insight into the structure of the infinite-width DNTK.
Then we prove that the integral operators over the unit sphere induced by operators from the same group orbit share the spectrum, and their reproducing kernel Hilbert spaces differ only by an orthogonal transformation, and that the operator‑fixed DNTK exhibits zonal symmetry on the unit sphere for symmetric operators. Thus the eigenvalue decay rate of the associated integral operator can be derived for shallow networks with RePU activation and zero‑bias initialization. 
\end{abstract}
\maketitle

\section{Introduction} 
The neural tangent kernel (NTK) \cite{Jacot:2018NTK} is a powerful tool for analyzing the training dynamics of wide neural networks. The spectral properties of the infinite-width limit kernel and the associated Gram matrix play a fundamental role in the study of approximation \cite{Bach:2017breaking}, generalization \cite{AroraDu:2019}, and optimization \cite{Du:2019deep} of the over-parameterized networks. The invariance of the NTK under orthogonal transformations is instrumental in the theoretical analysis of training of neural networks \cite{Cho:2009Kernel,Geifman:2020Similarity}. Cho et al \cite{Cho:2009Kernel} derive an explicit expression for the infinite NTK of two-layer networks with the RePU activation (without a bias term) built upon the symmetry of the NTK and homogeneity of RePU. Geifman et al. \cite{Geifman:2020Similarity} prove that the NTK of shallow ReLU networks without bias shares the eigenfunctions and eigenvalue decay rate as the Laplace kernel on the unit sphere based on the zonal property of the NTK. 
This line of analysis was later extended to deep neural networks (DNNs) by \cite{ChenXu:2020} and \cite{BiettiBach:2021}, which shows that the reproducing kernel Hilbert space (RKHS) of the deep NTK coincide with that of the shallow network, using the recursive structure and the arc‑cosine formulation of the shallow NTK. Karhadkar et al. \cite{Karhadkar:2024Bounds} derived bounds on the smallest eigenvalue of the NTK in terms of the directional distance between data points, using the symmetry of the NTK. 

Recently there is much interest in analyzing the optimization behavior of physics-informed neural networks (PINNs) \cite{Raissi:2019PINN}, motivated by the empirical successes across scientific problems \cite{HuJinZhou:2024,JinCai:2021,JinLiQuan:2024,CenJinLi:2025}. The differential neural tangent kernel (DNTK) \cite{JinWu:2026DNTK} was recently proposed to analyze the optimization of wide PINNs. It generalizes the NTK \cite{Jacot:2018NTK,Chizat:2019Lazy} to PINNs and provides a unified view of the NTK for PINNs. Several global convergence results of wide shallow PINNs for linear PDEs have been obtained within the DNTK framework, including deterministic algorithms \cite{Gao:2023GD,Xu:2024IGD,Xu:2024NGD} and stochastic algorithms \cite{JinWu:2025}. 
Suitable kernel matrices play a crucial role in analyzing the training in the lazy regime  \cite{Gao:2023GD,Xu:2024IGD,Xu:2024NGD,JinWu:2025}: the positivity of the infinite DNTK matrix provides the convergence guarantee of wide PINNs training for linear PDEs; the smallest eigenvalue governs the convergence rate, while the extreme eigenvalues influence the network width and the magnitude of the training rate. 
Further, the eigenvalue distribution dictates training efficiency and effectiveness: the high-frequency components of the solution are  difficult to learn for PINNs \cite{Wang:2022When,WangWang:2021eigen}. 
Thus the study of the DNTK is not only essential for revealing the training dynamics of PINNs but also for guiding the development of enhancement strategies. Unfortunately, the analysis of the DNTK is very involved due to the presence of multiple operators in PDEs. Recently, by exploiting the symmetry of the DNTK induced by the Laplace and identity operators, we derive explicit estimates for the extreme eigenvalues of the kernel matrix for shallow PINNs for the Poisson equation \cite{JinWu:2026spec}. 

In this work, we investigate the symmetry of the DNTK  under the orthogonal group and its implication for spectral analysis. We prove that the infinite DNTK for DNNs and the covariance of 
the infinite-width network output are invariant under the action of the orthogonal group. 
This holds for any activation function (provided that the infinite-width limit exists), irrespective of the presence of a bias term. Compared with the standard NTK, the main challenge of the DNTK lies in characterizing the symmetry, since the domain of the DNTK involves differential operators. Moreover, the DNTK with one fixed operator generally lacks the zonal property (except for some highly symmetric operators), which prevents directly transferring the symmetry results for the NTK. Nonetheless, the DNTK is unchanged when the orthogonal matrix acts simultaneously on the Euclidean points and differential operators. This group action on operators, 
which we interpret by identifying operators with their symbols, is rooted in the classical 
theory of group actions on polynomial rings in invariant theory \cite{Neusel:2007invariant}.

More specifically, we prove the symmetry of the DNTK for shallow neural networks (SNNs) and DNNs, separately. 
In the shallow case, we prove the symmetry for all components of the DNTK (including the covariance function) 
using its explicit expression, which arises from the weight and bias of the inner and outer layers. In the deep case, we prove the symmetry by mathematical induction on the network depth, using the recursive formulas for the covariance function and DNTK, in two different ways using different perspectives on the group action. The result for SNNs is more refined than that for DNNs and is not subsumed by the DNN analysis, and the employed analysis techniques in the proofs are completely distinct. In essence, the symmetry results stem from the invariance of the standard Gaussian vector under the orthogonal group. 
The symmetry of the DNTK implies that the integral operators on the unit sphere induced by operators from the same orbit under the group action are pairwise conjugate via an orthogonal matrix and hence share the spectrum, and their eigenfunctions and the associated RKHSs coincide up to the action of an orthogonal matrix.
For operators that reside in the invariant ring, e.g., the identity, the radial derivative, the Laplace operator and the ring they generate, the operator-fixed DNTKs are zonal on the unit sphere and thus computable for RePU SNNs with zero-bias initialization. Moreover, by exploiting the symmetry, we prove that the sum of operator-fixed DNTKs induced by the first order derivatives is also zonal and thus computable, which enables embedding the RKHS of the first order derivatives into a known one.

The rest of the article is organized as follows. In Section \ref{S:Preliminaries}, we review the infinite limit of the networks output and the DNTK, including covariance function and infinite DNTK for SNNs and the recursive formulas for DNNs, and the basics about the group action. In Section \ref{S:Symmetry}, we describe the group action of matrix on DNTK. Then, we prove the main results, namely, the symmetry of covariance function and infinite DNTK. In Section \ref{S:Conseq}, we discuss the property of the induced integral operators from the symmetry of DNTK. Throughout, the notation $[n]$ denotes the set $\{1,\ldots,n\}$ for an positive integer $n$ and \(\N=\{0,1,2,\cdots\}\) the set of all nonnegative integers.

\section{Preliminaries} \label{S:Preliminaries}

\subsection{Preliminaries of DNTK} The DNTK generalizes  the NTK to PINNs, and is useful for analyzing the optimization of wide PINNs. Now we describe the infinite-width limit of the NN output and the DNTK. See \cite{JinWu:2026DNTK} for more details.  Let \( f^{(L)} : \mathbb{R}^{d_0} \to \mathbb{R}^{d_L} \) with the input dimension $d_0=d$ and parameters \(\theta\) 
be a fully-connected NN with \(L-1\) hidden layers for some integer \(L\geq 1\), given recursively by 
\begin{equation}\label{E:L-layerNN}
\begin{aligned}
	f^{(1)}(x;\theta) & = W^{(1)} x + \gamma b^{(1)},\\
	f^{(\ell+1)}(x;\theta) & = \frac{1}{\sqrt{d_\ell}} W^{(\ell+1)} 
	\sigma \left( f^{(\ell)}(x;\theta) \right) + \gamma b^{(\ell+1)}, \quad \ell=1,\ldots, L-1,
\end{aligned}
\end{equation}
with the weight matrices $ W^{(\ell)} \in \mathbb{R}^{d_\ell \times d_{\ell-1}}$ and the bias vectors 
$b^{(\ell)} \in \mathbb{R}^{d_\ell}$. The nonlinear activation function \( \sigma : \mathbb{R} \to \mathbb{R} \) 
applies entrywise to vectors, and the weight \( \gamma \geq 0 \) of the bias is non-learnable, 
with \(\gamma=0\) corresponding to the case without a bias term. 

Let \(\FF_d:=\FF(\R^d)\) be the space of all scalar functions on \(\R^d\) and
\(\LD_d := \LD_d(\FF_d)\) the space of all linear differential operators
acting on functions from \(\R^d\) to \(\R\) with the coefficients in \(\FF_d\). 
Then, the DNN output \(f^{(L)}\) is a function on the product space \(\PLDR_d := \LD_d\times \R^d\)
\begin{equation}\label{E:DiffOut}
	f^{(L)}:  \PLDR_d \to \R^{d_L}, \quad
	  (\AA, x) \mapsto \AA f^{(L)}(x;\theta) 
	  := \AA_z f^{(L)}(z;\theta) \Big|_{z=x},
\end{equation}
with the operator $\mathcal{A}$ acting componentwise on the network output \(f^{(L)}\)\footnote{If the activation \(\sigma\) 
is not smooth but only differentiable up to some order, then we should accordingly restrict the maximum order of 
the differential operators in the domain of network \(\PLDR_d\), while all other arguments keep unchanged.}. 
The point \((\AA, x)\in \PLDR_d\), denoted by \(x^\AA\), is called the \emph{\(\AA\)-marked point \(x\)}. 
Given \(\mu,\nu\in[d_L]\) and $\AA,\BB\in \PLDR_d$, the DNTK is defined on the product space \(\PLDR_d\) and given by 
\[ \Theta^{(L)}_{\mu\nu}: \PLDR_d\times \PLDR_d\to \R, \quad 
	(x^\AA, y^\BB) \mapsto \llan \nabla_\theta \mathcal{A}f^{(L)}_\mu(x;\theta), \nabla_\theta\mathcal{B}f^{(L)}_\nu(y;\theta)  \rran. \]
Suppose that the activation function \(\sigma\) satisfies some mild growth assumption 
and all the parameters are initialized to the independent and identically distributed (i.i.d.) standard Gaussian distribution. 
Then, as the widths of all hidden layers tend to infinity, the components of the NN output \(f^{(L)}\) converge 
in distribution to a centered i.i.d. Gaussian process on the product space \(\PLDR_d\) with  the covariance function \(\Sigma^{(L)}_\infty\), 
and the DNTK \(\Theta^{(L)}_{\mu\nu}\) for all \(\mu,\nu\in[d_L]\) converge in probability to a deterministic kernel
\(\Theta^{(L)}_{\infty}\delta_{\mu\nu}\) on \(\PLDR_d\):
\begin{gather*}
	\lim_{d_1,\ldots d_{L-1}\to\infty} f_{\mu}^{(L)}\overset{d}{\longrightarrow} f_\infty^{(L)}
		= \mathcal{GP}\left(0,\Sigma^{(L)}_{\infty}\right), \quad \mu \in [d_L], \\
	\lim_{d_1,\ldots,d_{L-1}\to\infty} \Theta^{(L)}_{\mu\nu} \overset{p}{\longrightarrow} 
		\Theta_{\infty}^{(L)}\delta_{\mu\nu}, \quad \mu,\nu \in [d_L],
\end{gather*}
where $\delta_{\mu\nu}$ denotes the Kronecker delta symbol, i.e., $\delta_{\mu\nu} =1$ for  $\mu = \nu$ and  
$\delta_{\mu\nu} =0$ otherwise. 
Note that the assertion still holds if there is no hidden layer (\(L=1\)), namely, the components 
\(f_{\mu}^{(1)} =  f_{\infty}^{(1)}\) are already centered i.i.d. Gaussian process and 
\(\Theta^{(1)}_{\mu\nu} = \Theta^{(1)}_{\infty}\delta_{\mu\nu}\) are deterministic. 
Let \(\mathcal{N}(0,1)\) and \(\mathcal{N}(0,I_d)\) represent the standard Gaussian variable and 
standard Gaussian vector in \(\mathbb{R}^d\), respectively.  In the case of two-layer NNs (\(L=2\)), the covariance function 
\(\Sigma^{(2)}_{\infty}:\PLDR_d\times\PLDR_d\to \R\) is given by
	\begin{equation}\label{E:2-Cov}
		\Sigma^{(2)}_{\infty}\left(x^\AA,y^\BB\right) = \E_{u\sim \mathcal{N}(0,I_d), v\sim \mathcal{N}(0,1)}
		\left[ \mathcal{A}\sigma\left(u^\top x +  \gamma v\right)\mathcal{B}\sigma\left(u^\top y +  \gamma v\right) \right] 
	 	+ \gamma^2 a(x)b(y),
	\end{equation} 
where the operators $\AA$ and $\BB$ with zeroth order terms $a(z)$ and $b(z)$ respectively 
act on the input variable as in \eqref{E:DiffOut}; 
while the infinite DNTK \(\Theta^{(2)}_{\infty}:\PLDR_d\times\PLDR_d\to \R\) is given by
	\begin{equation}\label{E:2-DNTK}
	\begin{aligned}
		\Theta^{(2)}_{\infty}\left( x^\AA,y^\BB \right) 
		 =& \E_{u\sim \mathcal{N}(0,I_d), v\sim \mathcal{N}(0,1)}
		\left[ \mathcal{A}\sigma\left(u^\top x + \gamma v\right)
		\mathcal{B}\sigma\left(u^\top y + \gamma v\right) \right] + \gamma^2 a(x)b(y) \\
		&  + \E_{u\sim \mathcal{N}(0,I_d), v\sim \mathcal{N}(0,1)}
		\left[ \sum_{k=1}^{d}\mathcal{A}\left(\sigma'\left(u^\top x + \gamma v\right)x_k\right)
		\mathcal{B}\left(\sigma'\left(u^\top y + \gamma v\right)y_k\right)  \right] \\
		& + \gamma^2 \E_{u\sim \mathcal{N}(0,I_d), v\sim \mathcal{N}(0,1)}
		\left[ \mathcal{A}\sigma'\left(u^\top x + \gamma v\right)
		\mathcal{B}\sigma'\left(u^\top y + \gamma v\right)  \right],
	\end{aligned}
	\end{equation}
with \(x_k\) and \(y_k\) being the \(k\)th component of the inputs \(x\in\Omega\) and \(y\in\partial\Omega\), respectively. 
The four summands of \(\Theta^{(2)}_{\infty}\) arise from the gradient with respect to 
the outer weight, the outer bias, inner weight and inner bias parameters, respectively.
To describe the recursive formulas of the covariance function and the DNTK for DNNs, we need the 
\emph{Fa\`a di Bruno's formula}. Given a function \(h = h(z_1,\ldots,z_d):\R^d \to \R\) 
and a univariate function \(f:\R\to \R\), then for any multi-index 
\(\alpha = (\alpha_1, \ldots, \alpha_d)\), we have the following chain rule:
\begin{equation}\label{E:Faa-alpha}
	\partial^\alpha f(h)
	= \sum_{\pi \in \Pi_\alpha} f^{(|\pi|)}(h) \prod_{A \in \pi} \partial^A h
	= \sum_{\pi \in \Pi_\alpha} f^{(|\pi|)}(h) 
	  \prod_{A \in \pi} \frac{\partial^{|A|} h}{\prod_{j \in A} \partial z_j}, 
\end{equation}
where \(\Pi_\alpha\) is the set of all partitions of the set \(S_\alpha := 
\left\{ 1_{1},\ldots,1_{\alpha_1},2_{1},\ldots,2_{\alpha_2},\cdots, d_{1},\ldots,d_{\alpha_d} \right\}\)
with \(z_{i_j} = z_i\) in formula \eqref{E:Faa-alpha} for \(i\in[d]\) and \(j\in[\alpha_i]\) 
(if \(\alpha_i = 0\), then \([\alpha_i] = \varnothing\)); see \cite{JinWu:2026DNTK} for the detailed derivations.
For a linear operator \(\AA\in\LD_d\), let \(I_\AA\) denote the finite set of all multi-indices in \(\AA\)
and \(D\AA\) be the vector of \(d\) dimension consisting of coefficients of \(\AA\) corresponding to the first 
order indices \(\partial_{z_1},\ldots,\partial_{z_d}\), namely, 
\begin{equation}\label{E:LinDif}
	\AA = \sum_{\alpha \in I_\AA} a_\alpha \partial^{\alpha} \text{ with } a_\alpha\in \MM_d, \quad
	D\AA = [a_1, \cdots, a_d]^\top \text{ with } a_k := a_{\partial_{z_k}} \text{ for } k\in[d].
\end{equation}
Note that \(a_\alpha\) are not identically zero for \(\alpha\in I_\AA\) by convention, but \(a_k\) may be zero for \(k\in[d]\). Then  
the recursive formula for the covariance function \(\Sigma^{(L)}_\infty:\PLDR_d\times\PLDR_d\to \R\) 
is given inductively for \(\ell=1,\cdots,L-1\) by 
	\begin{equation}\label{E:deep-Cov}
		\begin{aligned}
			\Sigma^{(1)}_{\infty}\left( x^\AA,y^\BB \right) 
			& = \left(a(x)x+D\AA(x)\right)^\top \left(b(y)y+D\BB(y)\right) + \gamma^2 a(x)b(y), \\
			\Sigma^{(\ell+1)}_{\infty}\left( x^\AA,y^\BB \right) 
			 =& \E_{f^{(\ell)}_\infty } \left[ \left(\sum_{\alpha \in  I_\AA } a_\alpha (x) \sum_{\pi \in \Pi_\alpha} 
				\sigma^{(|\pi|)}\left( f^{(\ell)}_\infty  \left( x \right) \right) 
				\prod_{A \in \pi} f^{(\ell)}_\infty \left(x^{\partial^A}\right) \right) \cdot \right. \\
			& \left. \left( \sum_{\beta \in  I_\BB } b_\beta (y) \sum_{\lambda \in \Pi_\beta} 
				\sigma^{(|\lambda|)}\left( f^{(\ell)}_\infty  \left( y \right) \right) 
				\prod_{B \in \lambda} f^{(\ell)}_\infty \left(y^{\partial^B}\right) \right) \right]
			 	+ \gamma^2 a(x)b(y),
		\end{aligned}
	\end{equation}
where \(\AA = \sum_{\alpha \in  I_\AA } a_\alpha (z) \partial^{\alpha}\) and 
\(\BB = \sum_{\beta \in  I_\BB } b_\beta (z) \partial^{\beta}\)  
with the zeroth order terms \(a(z)\) and \(b(z)\), respectively. 
For the operator \(\AA = \sum_{\alpha \in I_\AA} a_\alpha \partial^{\alpha}\), define 
\begin{equation}\label{E:AA_sig(f_inf)}
\begin{aligned}
	\AA \sigma \left( f^{(\ell)}_{\infty}(x) \right) 
	  & = \sum_{\alpha \in  I_\AA } a_\alpha (x) \sum_{\pi \in \Pi_\alpha} 
		\sigma^{(|\pi|)}\left( f^{(\ell)}_\infty  \left( x \right) \right) 
		\prod_{A \in \pi} f^{(\ell)}_\infty \left(x^{\partial^A}\right).
\end{aligned}	
\end{equation}
The recursive formula for the infinite DNTK \(\Theta^{(L)}_\infty:\PLDR_d\times\PLDR_d\to \R\) 
is given inductively for \(\ell=1,\cdots,L-1\) by 
	\begin{equation}\label{E:deep-DNTK}
		\begin{aligned}
			\Theta^{(1)}_{\infty}\left( x^\AA,y^\BB \right) 
				& = \Sigma^{(1)}_{\infty}\left( x^\AA,y^\BB \right), \\
			\Theta^{(\ell+1)}_{\infty}\left( x^\AA,y^\BB \right)
				& = \Sigma^{(\ell+1)}_{\infty}\left( x^\AA,y^\BB \right) 
				    + \sum_{\alpha\in I_\AA }\sum_{\beta\in I_\BB } a_\alpha(x)b_\beta(y)\cdot \\
				& \qquad \sum_{\xi\leq\alpha}\sum_{\zeta\leq\beta}\binom{\alpha}{\xi}\binom{\beta}{\zeta}
					\dot{\Sigma}^{(\ell+1)}_{\infty}\left( x^{\partial^\xi},y^{\partial^\zeta} \right)
					\Theta^{(\ell)}_{\infty}\left( x^{\partial^{(\alpha-\xi)}},y^{\partial^{(\beta-\zeta)}} \right),
		\end{aligned}	
	\end{equation}
where the kernel \(\dot{\Sigma}^{(\ell+1)}_{\infty}\) takes the form as the expectation term in 
\(\Sigma^{(\ell+1)}_{\infty}\) by replacing the nonlinearity \(\sigma\) by \(\sigma'\):
	\begin{equation}\label{E:dotCov}
		\dot{\Sigma}^{(\ell+1)}_{\infty} \left( x^\AA,y^\BB \right) 
		= \E_{f^{(\ell)}_\infty} \left[ \AA \sigma' \left( f^{(\ell)}_{\infty}(x) \right)
		\BB \sigma' \left( f^{(\ell)}_{\infty}(y) \right) \right].
	\end{equation}
Similar to the shallow case, the two summands of \(\Sigma^{(\ell+1)}_{\infty}\) arise from the gradient 
with respect to the parameters at \((\ell+1)\)-th layer and layers below \((\ell+1)\)-th layer respectively.
By fixing the operators \(\AA=\BB\) in the domain of \(\Theta^{(L)}_\infty\), we obtain
a kernel on the Euclidean space \(\Theta^{(L)}_\infty[\AA,\AA]:\R^d\times\R^d \to \R\), 
referred to as the \emph{operator-fixed DNTK} or \(\AA\)-\emph{NTK} to explicitly indicate the operator $\AA$.
This kernel is the infinite-width limit of the DNTK induced by the differential of the NN \(\AA f^{(L)}:\R^d\to\R^{d_L}\).
Meanwhile, the limit \(\AA f^{(L)}_\infty:\R^d\to\R^{d_L}\) of \(\AA f^{(L)}\) is a Gaussian process on the 
Euclidean space \(\R^d\) with the covariance function given by \(\Sigma^{(L)}_\infty[\AA,\AA]:\R^d\times\R^d \to \R\).

\subsection{Group action}

Let $G$ be a group and $X$ a set. A \textbf{(left) action} of $G$ on $X$ is a map
\begin{equation}\label{E:G-action}
	G \times X \longrightarrow X, \quad (g, x) \longmapsto  g \cdot x,
\end{equation}
such that for all \(x\) there hold 
(i) \(e \cdot  x = x\) for the identity $e\in G$;
(ii) $ (g h)\cdot  x =  g \cdot (h \cdot x)$ for all $g, h \in G$.
\(X\) is called a \(G\)-\textbf{set} and sometimes the group action \(g\cdot x\) is denoted by \(g(x)\) or simply \(gx\).
Note that each element \(g\in G\) induces a bijective map from \(X\) to itself with the inverse map given by \(g^{-1}\).
Let \(X^*\) be the \textbf{dual set} of \(X\), i.e., the space consisting of all maps from \(X\) to \(\R\).
The action of \(G\) on \(X\) induces an action of \(G\) on \(X^*\) as follows.
Given any element \(\phi:X\to \R\) in \(X^*\) and \(g\in G\), the action of \(g\) at \(\phi\) is given by
\( (g,\phi)\mapsto g\cdot \phi \in X^*\) such that 
\begin{equation}\label{E:dual_set}
	g\cdot \phi:X \to \R, \quad g\cdot \phi(x) := \phi \circ g^{-1} (x) = \phi(g^{-1}\cdot x),
\end{equation}
where \(\circ\) represents the composition of functions. 
This map indeed defines an action on \(X^*\): (i) for the identity element, \(e\cdot \phi = \phi\) since 
\(e\cdot \phi(x) = \phi(e^{-1}\cdot x) = \phi(x)\) for all \(x\in X\); 
(ii) for any \(g,h\in G\), we have
\( g\cdot(h\cdot\phi) = g \cdot \left(\phi\circ h^{-1}\right) = 
\phi\circ h^{-1} \circ g^{-1} = \phi\circ (gh)^{-1} = (gh)\cdot \phi. \)

Let \(X\) be a \(G\)-set and \(Y\subset X\) a subset of \(X\).
If \(Y\) is closed under the group action, namely, for all \(g\in G\) and \(y\in Y\) there holds \(g\cdot y \in Y\), 
then the set \(Y\) is itself a \(G\)-set with the group action inherited from the action of \(G\) on \(X\), and \(Y\) is said to be a \(G\)-\textbf{subset} of \(X\). The \textbf{orbit} of $x\in X$ is a subset of \(X\) given by 
\begin{equation}\label{E:G-orbit}
	G\cdot x:= \{g\cdot x: g\in G\} \subset X.
\end{equation}
Any orbit is a \(G\)-subset and \(X\) is a disjoint union of all its orbits.
The \textbf{invariant set} of \(X\) under the action of \(G\), denoted by \(X^G\), is defined by 
\begin{equation}\label{E:G-invariant}
	X^G:= \{ x\in X: g\cdot x = x, \forall g\in G \}.
\end{equation}
Finally, the Cartesian product of \(G\)-sets is again a \(G\)-set. Specifically, let \(X\) and \(Y\) be two \(G\)-sets.
Then the group \(G\) acts on the product set \(X\times Y\)  componentwise:
\begin{equation}\label{E:G-product}
	G \times (X\times Y) \longrightarrow X\times Y, \quad g\cdot (x,y) \longmapsto  (g\cdot x, g\cdot x).
\end{equation}

Next we specify the action of the orthogonal group on the domain of DNTK. See \cite{Hungerford:2012Algebra} for the algebraic aspect. Recall that the domain of DNTK is the product space \(\PLDR_d := \LD_d\times \R^d\). Next we illustrate the group action on these factors. The first one is the Euclidean space acted by the matrix through multiplication.

\begin{proposition}\label{Exm:G-Euclidean}
Let \(X=\R^d\) be the Euclidean space and 
\(G\subset \mathrm{GL}(d;\R)\) a subgroup of the general linear group (the set of all invertible matrices). 
Let \(\FF_d=\FF(\R^d)\) be the set of all maps from \(\R^d\) to \(\R\). 
\begin{itemize}
\item[(i)] The group \(G\) acts on \(\R^d\) by the multiplication between matrices and vectors: \(g\cdot x = gx\)
for \(g\in G\) and \(x\in \R^d\). This action satisfies condition (i) 
since the identity matrix acts trivially on vectors and condition (ii) by the associative
law of matrix multiplication. 

\item[(ii)] Note that \(\FF_d\) can be viewed as the dual set of the Euclidean space \(\R^d\) 
and thus has an action by \(G\) through formula \eqref{E:dual_set}. 
Moreover, \(\FF_d\) is a unital commutative ring with the identity element being the constant function \(1\) 
and the addition/multiplication given by the addition/multiplication between functions.
Note that for any \(g\in G\subset \mathrm{GL}(d;\R)\) and all maps \(\phi,\psi\in \FF_d\), the 
action of \(g\) at the summation and product functions \(\phi+\psi, \phi\psi\) satisfies 
\begin{align*}
	g\cdot (\phi + \psi) & = (\phi + \psi)\circ g^{-1} = \phi\circ g^{-1} + \psi\circ g^{-1}
	=  g\cdot \phi + g\cdot \psi, \\
	g\cdot (\phi\psi) & = (\phi\psi)\circ g^{-1} = (\phi\circ g^{-1})(\psi\circ g^{-1}) 
	=  (g\cdot \phi)(g\cdot \psi).
\end{align*}
Therefore, the group action is consistent with the ring structure of \(\FF_d\). Thus,
the group \(G\) is said to acts on the ring \(\FF_d\) and \(\FF_d\) is said to be a \(G\)-\textbf{ring}.
\end{itemize}
\end{proposition}

The next one is the polynomial ring acted by the general linear group, 
which comes from invariant theory and we follow the approach in \cite[Chapter 3]{Neusel:2007invariant}.
\begin{proposition}\label{Exm:G-Poly_Ring}
Let \(V\) be a real vector space of dimension \(d\) and \(G\subset \mathrm{GL}(V)\) 
a subgroup of the automorphism group of \(V\).
Let \(\mathrm{PR}(d;\R) = \R[\xi_1,\ldots,\xi_d]\) be the ring of polynomials in \(d\) indeterminates \(\{\xi_1,\ldots,\xi_d\}\) 
over the real field \(\R\). Let \(\mathrm{PR}_{(n)} = \R[\xi_1,\ldots,\xi_d]_{(n)}\subset \R[\xi_1,\ldots,\xi_d]\) 
be the set of all homogeneous polynomials 
of degree \(n\) for \(n\in \mathbb{N} = \{0,1,2,\ldots\}\). Then each \(\mathrm{PR}_{(n)}\) is an additive
commutative group. Any element \(p\in \mathrm{PR}(d;\R)\) can be uniquely written as a finite sum 
\[ p = \sum_{i\in I} p_i,\quad \text{with } p_i\in \mathrm{PR}_{(i)}. \]
Moreover, the product between polynomials satisfies \(p_ip_j\in \mathrm{PR}_{(i+j)}\)
for all \(i,j\in \mathbb{N}\) and \(p_i\in \mathrm{PR}_{(i)},p_j\in \mathrm{PR}_{(j)}\). In other words, 
the polynomial ring \(\mathrm{PR}(d;\R) = \R[\xi_1,\ldots,\xi_d]\) is an \textbf{\(\mathbb{N}\)-graded ring}: 
\[ \mathrm{PR}(d;\R) = \R[\xi_1,\ldots,\xi_d] = \bigoplus_{n\in \mathbb{N}} \R[\xi_1,\ldots,\xi_d]_{(n)}
 = \bigoplus_{n\in \mathbb{N}} \mathrm{PR}_{(n)}. \]
Fix a basis of the vector space \(V = \mathrm{span}\{e_1,\ldots,e_d\}\) and let \(\{\xi_1,\ldots,\xi_d\}\) 
be the dual basis of \(\{e_1,\ldots,e_d\}\) in the linear functional space \(V^*\), i.e., 
$\xi_i(e_j) = \delta_{ij}.$ For any element \(g\in G\), the action of \(g\) at \(\xi_k\) for \(k\in[d]\) is defined by
\begin{equation}\label{E:G_PR_basis}
	g\cdot \xi_k  = \xi_k\circ g^{-1}:V\to \R, \quad v\mapsto  \xi_k(g^{-1}v).
\end{equation}
Moreover, there holds
\[ g\cdot \xi_k(\mu_1v_1 + \mu_2v_2) = \mu_1g\cdot \xi_k(v_1) + \mu_2g\cdot \xi_k(v_2),\quad 
\forall \mu_1,\mu_2\in\R, v_1,v_2\in V.
\]
Thus \(g\cdot \xi_k\) lies in the dual space \(V^* = \mathrm{span}\{\xi_1,\ldots,\xi_d\}\)
and is a linear sum of the basis \(\{\xi_1,\ldots,\xi_d\}\). 
This action can be extended linearly to the whole space \(V^*\) and
multiplicatively to the polynomial ring \(\mathrm{PR}(d;\R)\). We show the latter assertion directly.
Given a monomial \(\xi^\alpha = \xi_1^{\alpha_1}\cdots \xi_d^{\alpha_d}\), with \(\xi=(\xi_1,\ldots,\xi_d)\)
and \(\alpha = (\alpha_1,\ldots,\alpha_d)\) is a multi-index, the action of \(g\) on \(\xi^\alpha\) is given by 
\[ g\cdot \xi^\alpha = g\cdot (\xi_1^{\alpha_1}\cdots \xi_d^{\alpha_d})
:= (g\cdot \xi_1)^{\alpha_1}\cdots (g\cdot \xi_d)^{\alpha_d}. \]
Since each \(g\cdot \xi_i\) is a linear sum of \(\{\xi_1,\ldots,\xi_d\}\), we have 
\(g\cdot \xi^\alpha \in \mathrm{PR}(d;\R)\) and \(g\cdot \xi^\alpha\) has the same degree 
\(|\alpha|\) as the monomial \(\xi^\alpha\). For a generic element \(p(\xi)=\sum_{\alpha\in I}a_\alpha \xi^\alpha \) 
in \(\mathrm{PR}(d;\R)\), where \(I\) is a finite set of indices and \(a_\alpha\in \R\) is 
the coefficients of the polynomial, the action of an element \(g\) on \(p(\xi)\) is given by 
\begin{equation}\label{E:G-Poly_Ring}
	g\cdot p(\xi) = g\cdot \sum_{\alpha\in I}a_\alpha \xi^\alpha 
	= \sum_{\alpha\in I}a_\alpha (g\cdot \xi^\alpha).
\end{equation}
This action of \(G\) on the polynomial ring \(\mathrm{PR}(d;\R)\) is well-defined and 
actually a \textbf{graded ring homomorphism}: for all \(g\in G\) and \(p,q \in \mathrm{PR}(d;\R)\),
\begin{equation}\label{E:gra-ring-hom}
	g\cdot(p+q) = g\cdot p + g\cdot q, \quad g\cdot(pq) = (g\cdot p)(g\cdot q), \quad 
	G\cdot \mathrm{PR}_{(n)}\subset \mathrm{PR}_{(n)}, \quad  \forall n\in \mathbb{N}.
\end{equation}  	
\end{proposition}

The algebraic approach of the action of general linear group on the polynomial ring is intrinsic but inconvenient to use. Thus, we give an explicit formula of the group action through the view of group action on the dual set \eqref{E:dual_set}. 
First, the real vector space \(V\cong \R^d\) after fixing some basis of \(V\) and thus \(G\subset \mathrm{GL}(d;\R)\).
Then we view each \(\xi_i\) as the coordinate function on \(\R^d\) such that 
\( \xi_i:\R^d \to \R\) with \( \xi_i(v) = v_i \) for \(v=(v_1\ldots,v_d)\).
Hence the elements in \(\mathrm{PR}(d;\R) = \R[\xi_1,\ldots,\xi_d]\) are polynomial functions in \(\mathbb{R}^d\) 
and is a subset of the \(G\)-set \(\FF(\R^d)\).

\begin{lemma}\label{L:G-PL_formula} 
	The polynomial ring \(\mathrm{PR}(d;\R)\) is a \(G\)-subring of \(\FF(\R^d)\) with the inherited action coincides with that in Proposition \ref{Exm:G-Poly_Ring}. 
	Let \(\xi=[\xi_1,\ldots,\xi_d]^\top\) and \(g\xi\) be the multiplication between a matrix and a vector for \(g\in G\).
	Then the action of \(G\) on \(\mathrm{PR}(d;\R)\) for \(g\in G\) and 
	\(p(\xi)=\sum_{\alpha \in I}a_\alpha \xi^\alpha\in \mathrm{PR}(d;\R)\) is given by 
	\begin{equation*}
	\begin{aligned}
		g\cdot p(\xi)  = g \cdot \left(\sum_{\alpha \in I}a_\alpha \xi^\alpha\right) 
		 = \sum_{\alpha \in I} a_\alpha (g^{-1}\xi)^\alpha = p(g^{-1}\xi).
	\end{aligned}
	\end{equation*}
\end{lemma}

\begin{proof} 
First, the polynomial functions \(\mathrm{PR}(d;\R)\) is a subring of \(\FF(\R^d)\).
Next we prove that \(\mathrm{PR}(d;\R)\) is a \(G\)-subring of \(\FF(\R^d)\) endowed with 
the action given in Proposition \ref{Exm:G-Euclidean}. 
For \(g\in G\) and \(\xi_k\) with \(k\in[d]\), we have 
\(g\cdot \xi_k = \xi_k\circ g^{-1}:\R^d\to \R\) such that 
\[ g\cdot \xi_k (v) =  \xi_k(g^{-1}v) = g^{-1}_k v = (g^{-1}_k \xi)(v), \quad \forall v\in\R^d, \] 
with \(g^{-1}_k\) denoting the \(k\)th row of the matrix \(g^{-1}\). Thus, we have 
\(g\cdot \xi_k = g^{-1}_k \xi \in \mathrm{PR}(d;\R)\), which is linear sum of the functions 
\(\{\xi_1,\ldots,\xi_d\}\). 
Then, for a general element \(p(\xi) = \sum_{\alpha \in I}a_\alpha \xi^\alpha\) and any \(v\in \R^d\), we have
\begin{align*}
	g\cdot p(\xi) (v) & = p(\xi)(g^{-1}v) = \sum_{\alpha \in I}a_\alpha \xi^\alpha (g^{-1}v)
	= \sum_{\alpha \in I}a_\alpha  (g^{-1}v)^\alpha
	= \sum_{\alpha \in I}a_\alpha (g^{-1}\xi)^\alpha(v)
	= p(g^{-1}\xi)(v).
\end{align*}
Thus we have 
$$ g\cdot p(\xi) = p(g^{-1}\xi)\in\mathrm{PR}(d;\R), \quad \forall p(\xi)\in\mathrm{PR}(d;\R),$$ 
which also implies the consistency of the group action and the ring structure of \(\mathrm{PR}(d;\R)\) in Proposition \ref{Exm:G-Euclidean} (ii).

Finally, we prove that the group action on \(\mathrm{PR}(d;\R)\) inherited from \(\FF(\R^d)\) coincides with that in Proposition \ref{Exm:G-Poly_Ring}. For the vector \(\xi=[\xi_1,\ldots,\xi_d]^\top\) consisting of the coordinate functions,
we have \(g\cdot\xi=g^{-1}\xi\) since 
$$ g\cdot\xi(v) = \xi(g^{-1}v) = g^{-1}v, \quad  \forall v\in \R^d.$$
Thus, for \(p(\xi) = \sum_{\alpha \in I}a_\alpha \xi^\alpha\), we have 
\[ g\cdot p(\xi) = p(g^{-1}\xi) = \sum_{\alpha \in I}a_\alpha (g^{-1}\xi)^\alpha 
	= \sum_{\alpha \in I}a_\alpha (g\cdot\xi)^\alpha, \]
which reproduces the formula \eqref{E:G-Poly_Ring} in Proposition \ref{Exm:G-Poly_Ring}.
\end{proof}

\section{Symmetry of DNTK} \label{S:Symmetry}
In this section, we describe the main symmetry results. First we describe the action of a general linear group \(\mathrm{GL}(d;\R)\) on the DNTK 
\(\Theta^{(L)}_{\infty}:\PLDR_d\times\PLDR_d\to \R\). 
Let \(G\subset \mathrm{GL}(d;\R)\) be a subgroup of invertible linear transformations. 
Recall that the space of linear differential operators \(\LD_d = \LD_d(\FF_d)\) is a \(G\)-ring, where \(\FF_d = \FF(\R^d)\), the set of all functions on \(\R^d\). 
For generality, let \(\LD_d(\mathfrak{R})\) be the space consisting of operators whose coefficients are from a unital \(G\)-subring \(\mathfrak{R}\) of \(\FF(\R^d)\) that contains the coordinate functions \(\{z_1,\ldots,z_d\}\). 
The unital \(G\)-subring means that \(\mathfrak{R} \subset \FF(\R^d)\)
is itself a ring with the identity element and also a \(G\)-subset of \(\FF(\R^d)\). 
Examples include the set of constant functions (which is just \(\R\)),
continuous functions, bounded functions, smooth functions, functions differentiable up to a given order and \(\FF(\R^d)\). 
Then \(G\) can act on the space \(\LD_d(\mathfrak{R})\) through the total symbol of differential operators as follows. 
First, we construct the polynomial ring \(\mathrm{PR}(d;\mathfrak{R}) := \mathfrak{R}[\xi_1,\ldots,\xi_d]\) 
in \(d\) indeterminates \(\{\xi_1,\ldots,\xi_d\}\) over the ring \(\mathfrak{R}\), 
which is a unital commutative \(\mathbb{N}\)-graded ring \cite[Theorem 5.1]{Hungerford:2012Algebra}.
Note that \(G\) has actions both on the ring \(\mathfrak{R}\), cf. Proposition \ref{Exm:G-Euclidean} and the 
polynomial ring \(\mathrm{PR}(d;\R)\), cf. Proposition \ref{Exm:G-Poly_Ring}. 
We can define the action of \(G\) on \(\mathrm{PR}(d;\mathfrak{R})\) for \(g\in G\) and 
\(p(\xi) = \sum_{\alpha\in I} a_\alpha \xi^\alpha \in \mathrm{PR}(d;\mathfrak{R})\) as
\begin{equation}\label{E:G-DO} 
	g\cdot p(\xi) = g\cdot \sum_{\alpha\in I} a_\alpha \xi^\alpha 
	:= \sum_{\alpha} (g\cdot a_\alpha) (g\cdot \xi^\alpha). 
\end{equation}
From the identity \eqref{E:G-DO},  the action is well-defined and that each element in \(G\) induces 
a graded ring homomorphism on \(\mathrm{PR}(d;\mathfrak{R})\) as in Proposition \ref{Exm:G-Poly_Ring}. 
Second, the \textbf{(total) symbol} for an operator \(\AA\in \LD_d(\mathfrak{R})\), denoted by \(\mathrm{sbl}(\AA)\) and lying in the ring \(\mathrm{PR}(d;\mathfrak{R})\), is defined by replacing the partial operator \(\partial\) therein
with the indeterminates \(\xi\): 
\begin{equation}\label{E:symbol}
	\AA = \sum_{\alpha \in I_\AA} a_\alpha \partial^\alpha\in \LD_d(\mathfrak{R}) \quad \Longleftrightarrow \quad 
	\mathrm{sbl}(\AA) = \sum_{\alpha \in I_\AA} a_\alpha \xi^\alpha \in \mathrm{PR}(d;\mathfrak{R}), \quad a_\alpha\in \mathfrak{R}.
\end{equation}
It identifies the linear operator space \(\LD_d(\mathfrak{R})\) 
as the polynomial ring \(\mathrm{PR}(d;\mathfrak{R})\), and the action of \(G\) on \(\LD_d(\mathfrak{R})\) inherits from the action of \(G\) on \(\mathrm{PR}(d;\mathfrak{R})\). 
Finally, since the group \(G\subset \mathrm{GL}(d;\R)\) acts on both \(\LD_d(\mathfrak{R})\) and the Euclidean space \(\R^d\), it acts on the Cartesian space \(\PLDR_d(\mathfrak{R}):= \LD_d(\mathfrak{R}) \times \R^d\) and also the domain of DNTK \(\PLDR_d(\mathfrak{R})\times \PLDR_d(\mathfrak{R})\) componentwise by the formula \eqref{E:G-product}.  
Thus, \(G\) acts on DNTK \(\Theta^{(L)}_\infty\) as a function on \(\PLDR_d(\mathfrak{R})\times \PLDR_d(\mathfrak{R})\) by the formula \eqref{E:dual_set}.
Let \(\partial=[\partial_1,\ldots,\partial_d]^\top\) be the gradient operator 
and \(g\partial\) the multiplication between matrix and vector for \(g\in G\).
From the identification \eqref{E:symbol} and Lemma \ref{L:G-PL_formula}, the action of \(g^{-1}\in G\subset \mathrm{GL}(d;\R)\) 
on the element \(\AA\in \LD_d(\mathfrak{R})\) is given by
\begin{equation}\label{E:G-DO_formula}
	\AA = \sum_{\alpha \in I_\AA} a_\alpha \partial^\alpha \quad \Longrightarrow \quad 
	g^{-1}\cdot \AA = \sum_{\alpha \in I_\AA} (g^{-1}\cdot a_\alpha) (g^{-1}\cdot \partial)^\alpha
	= \sum_{\alpha \in I_\AA} (a_\alpha \circ g) (g \partial)^\alpha \in \LD_d(\mathfrak{R}).
\end{equation}
Thus, the action of \(g\in G\) on the DNTK is given by 
\(g\cdot \Theta^{(L)}_\infty = \Theta^{(L)}_\infty\circ g^{-1}: \PLDR_d(\mathfrak{R})\times \PLDR_d(\mathfrak{R}) \to \R\), 
such that for all \(x^\AA, y^\BB \in \PLDR_d(\mathfrak{R})\), we have
\begin{equation} \label{E:G-DNTK}
\begin{aligned}
	g\cdot \Theta^{(L)}_\infty(x^\AA, y^\BB) = \Theta^{(L)}_\infty\left(g^{-1}\cdot x^\AA, g^{-1}\cdot y^\BB\right)
	= \Theta^{(L)}_\infty\left((g^{-1}x)^{g^{-1}\cdot\AA}, (g^{-1}y)^{g^{-1}\cdot\BB}\right).
\end{aligned}	
\end{equation}

\begin{remark}
    By identifying operators with their symbols, the product \(\AA\cdot\BB\) for \(\AA,\BB\in \LD_d(\mathfrak{R}) = \mathrm{PR}(d;\mathfrak{R})\) represents the operator with symbol given by \(\mathrm{sbl}(\AA)\cdot \mathrm{sbl}(\BB)\) and is commutative in \(\LD_d(\mathfrak{R})\).
	This product between symbols should be well-distinguished with 
	the composition of operators, denoted by \(\AA\circ\BB\), as in Lemma \ref{L:AA-z_k}.
	Although the composition can be regarded as a product operation for linear operators, it may not be well-defined in \(\LD_d(\mathfrak{R})\) if the functions in \(\mathfrak{R}\) is not smooth enough.
	Even if the composition makes sense, it is generally not commutative. 
	For example, the composition between radial derivative operator \(\RR\) and 
	the Laplace operator \(\LL\), defined in equation \eqref{E:RR_LL} below, are not commutative:
	\begin{align*}
		\RR \circ \LL & = \left(\sum_{i=1}^{d} x_i \partial_{i}\right)
		\left(\sum_{j=1}^{d} \partial_{j}^2\right) = \sum_{i,j=1}^{d} x_i \partial_{i}\partial_{j}^2,\\
		\LL \circ \RR & =\left(\sum_{j=1}^{d} \partial_{j}^2\right) \left(\sum_{i=1}^{d} x_i \partial_{i}\right)
		= \sum_{j=1}^{d} \sum_{i=1}^{d} \partial_{j}^2(x_i \partial_{i}) 
		= \sum_{j=1}^{d} \left(\partial_{j}^2(x_j \partial_{j}) + \sum_{i\neq j}x_i\partial_{j}^2\partial_{i}\right) \\
		& = \sum_{j=1}^{d} \left( 2\partial_{j}^2 + x_j\partial_{j}^3 + \sum_{i\neq j}x_i\partial_{i}\partial_{j}^2 \right)
		= 2\sum_{j=1}^{d} \partial_{j}^2 + \sum_{i,j=1}^{d} x_i \partial_{i}\partial_{j}^2,
	\end{align*}
	which shows that \(\RR \cdot \LL = \LL\cdot\RR = \mathrm{sbl}(\RR \circ \LL)\neq \mathrm{sbl}(\LL\circ \RR)\).
\end{remark}

The orthogonal group \(\OO(d)\) is the set of all 
orthogonal matrices \(g\in \R^{d\times d}\) such that \(g^\top g = I_d\).

\begin{theorem}\label{T:Sym-SNN}
Let \(\OO(d)\) be the group of orthogonal matrices in $\mathbb{R}^d$.  Let \(\mathfrak{R}\) be a unital \(\OO(d)\)-subring of \(\FF(\R^d)\) and \(\PLDR_d(\mathfrak{R}) := \LD_d(\mathfrak{R}) \times \R^d\). Then the infinite DNTK \(\Theta^{(2)}_{\infty}\) \eqref{E:2-DNTK} for two-layer NNs is invariant under the action of the orthogonal group \(\OO(d)\): for any \(g\in\OO(d)\), we have 
	\(g\cdot \Theta^{(2)}_{\infty} = \Theta^{(2)}_{\infty}:\PLDR_d(\mathfrak{R})\times \PLDR_d(\mathfrak{R}) \to \R\): 
	\begin{equation*}
		\Theta^{(2)}_{\infty}(x^\AA, y^\BB) = \Theta^{(2)}_{\infty}\left(g^{-1}\cdot x^\AA, g^{-1}\cdot y^\BB\right)
		= \Theta^{(2)}_{\infty}\left((g^{-1}x)^{g^{-1}\cdot\AA}, (g^{-1}y)^{g^{-1}\cdot\BB}\right),\quad 
        \forall x^\AA, y^\BB\in \PLDR_d(\mathfrak{R}).
	\end{equation*}
	In addition, all the four summands in the expression \eqref{E:2-DNTK} of \(\Theta^{(2)}_{\infty}\), 
	in particular the covariance function \(\Sigma^{(2)}_{\infty}\), are invariant under the action of \(\OO(d)\).
\end{theorem}

\begin{theorem}\label{T:Sym-DNN}
	Let \(\OO(d)\) be the group of orthogonal matrices in $\mathbb{R}^d$. 
	Let \(\mathfrak{R}\) be a unital \(\OO(d)\)-subring of \(\FF(\R^d)\) and \(\PLDR_d(\mathfrak{R}) := \LD_d(\mathfrak{R}) \times \R^d\). 
	Then the infinite DNTK \(\Theta^{(L)}_{\infty}\) \eqref{E:deep-DNTK} for DNNs with \(L\geq 1\)
	is invariant under the action of the orthogonal group \(\OO(d)\): for any \(g\in\OO(d)\), we have 
	\(g\cdot \Theta^{(L)}_{\infty} = \Theta^{(L)}_{\infty}:\PLDR_d(\mathfrak{R})\times \PLDR_d(\mathfrak{R}) \to \R\): 
	\begin{equation*}
		\Theta^{(L)}_{\infty}(x^\AA, y^\BB) = \Theta^{(L)}_{\infty}\left(g^{-1}\cdot x^\AA, g^{-1}\cdot y^\BB\right)
		= \Theta^{(L)}_{\infty}\left((g^{-1}x)^{g^{-1}\cdot\AA}, (g^{-1}y)^{g^{-1}\cdot\BB}\right),\quad 	\forall x^\AA, y^\BB\in \PLDR_d(\mathfrak{R}).
	\end{equation*}
	In addition, the covariance function \(\Sigma^{(L)}_{\infty}\) is invariant under the action of \(\OO(d)\).
\end{theorem}

\begin{remark}
The invariance of DNTK originates from the symmetry of the standard Gaussian vector used in the initialization of network parameters.
	The results in Theorem \ref{T:Sym-SNN}, which shows the symmetry of all summands of \(\Theta^{(2)}_{\infty}\), 
	are not all included by Theorem \ref{T:Sym-DNN}.
	The invariance of the covariance function \(\Sigma^{(L)}_{\infty}\) under the orthogonal group \(\OO(d)\) 
	means that the Gaussian process \(f^{(L)}_\infty\) on \(\PLDR_d(\mathfrak{R})\) is invariant under the group \(\OO(d)\) 
	(or stationary Gaussian process \cite{Azangulov:2024I,Azangulov:2024II}):
	\[ g\cdot f^{(L)}_\infty = f^{(L)}_\infty \quad \Longleftrightarrow \quad 
		\GP\left(0,g\cdot\Sigma^{(L)}_{\infty}\right) = \GP\left(0,\Sigma^{(L)}_{\infty}\right),\quad 
\forall g\in\OO(d). \]
\end{remark}

\begin{remark}
    If we restrict the maximum order, say \(m\), of the operators in \(\LD_d(\mathfrak{R})\), the polynomials space is \(\mathrm{PR}(d;\mathfrak{R})_{\leq m}:= \bigoplus_{n=0}^{m} \mathfrak{R}[\xi_1,\ldots,\xi_d]_{(n)}\), which is not a ring since the multiplication is not closed 
	but instead a module over the ring \(\mathfrak{R}\). Nonetheless, the whole theory are not  
	affected except the ring structure of \(\mathrm{PR}(d;\mathfrak{R})_{\leq m}\) replaced by the module structure. Similarly, the ring structure of the coefficient set \(\mathfrak{R}\), which is not owned by the space of integrable functions, is not necessary to Theorems \ref{T:Sym-SNN} and \ref{T:Sym-DNN}.
\end{remark}

\section{Implications of symmetry} \label{S:Conseq}

Now using the symmetry for the operator-fixed DNTK, we derive several spectral results for the DNTK. 
Given \(\AA\in\LD_d(\mathfrak{R})\), we use \(\Theta^{(L)}_\AA\) to denote the kernel 
\(\Theta^{(L)}_\infty[\AA,\AA]:\R^d\times\R^d \to \R\). 
We restrict the domain of such \(\AA\)-NTKs to be the unite sphere \((\mathbb{S}^{d-1},{\rm d}S^{d-1})\) 
with \({\rm d}S^{d-1}\) the induced measure by the Lebesgue measure, and the coefficient ring \(\mathfrak{R}\) to be 
the space of continuous functions, i.e., \(\mathfrak{R}=C(\mathbb{S}^{d-1})\).
Then, for each \(\AA\in\LD_d(C(\mathbb{S}^{d-1}))\), we obtain a continuous kernel on the unit sphere $\mathbb{S}^{d-1}$:
\begin{equation*}
	\Theta^{(L)}_\AA:\mathbb{S}^{d-1}\times \mathbb{S}^{d-1} \to \R, \quad (x,y)\mapsto \Theta^{(L)}_\infty(x^\AA,y^\AA).
\end{equation*}
Let \(L^2(\mathbb{S}^{d-1}):=L^2(\mathbb{S}^{d-1};{\rm d}S^{d-1})\) be the
real-valued square integrable functions on \(\mathbb{S}^{d-1}\). 
Then the induced integral operator associated with \(\AA\)-NTK \(\Theta^{(L)}_\AA\) is defined by 
\begin{equation*}
	T^{(L)}_\AA: L^2(\mathbb{S}^{d-1})\to L^2(\mathbb{S}^{d-1}), \quad
	f \mapsto \int_{\mathbb{S}^{d-1}} \Theta^{(L)}_\AA(\cdot,y)f(y) ~ {\rm d}S^{d-1}(y).
\end{equation*}
By the compactness of \(\mathbb{S}^{d-1}\), the continuity of the coefficients of \(\AA\), and 
the positivity of the DNTK \cite{JinWu:2026DNTK}, the integral operator \(T^{(L)}_\AA\) has 
Mercer's decomposition \cite{Minh:2006Mercer}: \(T_\AA\) is a Hilbert-Schmidt operator with nonnegative real
spectrum \(\sigma(T_\AA) = \{\lambda_j\}_{j\in \N}\) such that \(\sum_{j\in\N}\lambda_j^2<\infty\) 
and the corresponding eigenfunctions \(\{\phi_j\}_{j\in \N}\) form an orthonormal basis of
\(L^2(\mathbb{S}^{d-1})\). Moreover, we have 
\begin{equation*}
	\Theta^{(L)}_\AA(x,y) = \sum_{j\in \N} \lambda_j \phi_j(x)\phi_j(y), \quad x,y\in \mathbb{S}^{d-1},
\end{equation*}
where the series converges absolutely for each pair on \(\mathbb{S}^{d-1}\) and uniformly on each 
compact subset of \(\mathbb{S}^{d-1}\). The RKHS of the kernel \(\Theta^{(L)}_\AA\) is given by 
\begin{equation}\label{E:RKHS}
	\HH^{(L)}_\AA := \HH_{\Theta^{(L)}_\AA} = \left\{ f = \sum_{j\in \N, \lambda_j>0} f_j\phi_j ~\Big|~ f_j\in\R, 
	\sum_{j\in \N, \lambda_j>0} f_j^2\lambda_j^{-1} < \infty  \right\} \subset L^2(\mathbb{S}^{d-1}).
\end{equation}

\begin{proposition}\label{P:AA=g-BB}
If \(\AA, \BB\in \LD_d(C(\mathbb{S}^{d-1}))\) are in the same orbit under the action of the group \(\OO(d)\), then the integral operators \(T^{(L)}_\AA\) and \(T^{(L)}_\BB\) are conjugate
	as the operators on \(L^2(\mathbb{S}^{d-1})\). Specifically, 
	let \(M_g:L^2(\mathbb{S}^{d-1})\to L^2(\mathbb{S}^{d-1})\) be the operator induced by 
	\(g\) such that \(f \mapsto g\cdot f = f\circ g^{-1}\). Then for \(\BB=g\cdot\AA\) with \(g\in\OO(d)\), we have
	\[ T^{(L)}_\BB = T^{(L)}_{g\cdot\AA} = M_{g} T^{(L)}_\AA M_{g^{-1}} = M_g T^{(L)}_\AA M_g^{-1}. \] 
	In particular, \(T^{(L)}_\AA\) and \(T^{(L)}_\BB\) share spectrum: 
	if \(\lambda\in \sigma(T_\AA)\) is an eigenvalue 
	of \(T_\AA\) with the  eigenfunction \(\phi \in L^2(\mathbb{S}^{d-1})\), then
	\(\lambda\in \sigma(T_\BB)\) with the eigenfunction \(g\cdot \phi = \phi\circ g^{-1}\) 
	and thus \(\HH^{(L)}_\BB = g\cdot \HH^{(L)}_\AA\).
\end{proposition}

All the first order derivatives \(\partial_k\) for \(k\in[d]\) are in the same orbit. 
In fact, let \(P_{ij}\) be the permutation matrix from the identity matrix by interchanging the \(i\)-th row 
and \(j\)-th row for \(i,j\in[d]\). Then, for the operator \(\partial_i\), we have 
\( P_{ij}\cdot \partial_i = (P_{ij}\partial)_i = \partial_j\). Thus all the integral operators 
\(T^{(L)}_{\partial_k}\) have the same spectrum and the RKHSs 
\(\HH^{(L)}_{\partial_k}\) only differ by a rotation.

Recall that a kernel on the unit sphere \(\mathbb{S}^{d-1}\) is \textbf{zonal} 
if the kernel value on the pair \((x,y)\) only depends on the inner product \(x^\top y\). 
In fact, the zonal kernel is equivalent to the kernels invariant under the orthogonal group. 
It is obvious the zonal property implies the invariance, and we prove the reverse direction next.

\begin{proposition}\label{P:zonal}
	If \(\AA\in\LD_d(C(\mathbb{S}^{d-1}))^{\OO(d)}\) lies in the ring of invariants under the action of \(\OO(d)\), i.e., \(g\cdot \AA = \AA\) for all \(g\in\OO(d)\), 
	then the \(\AA\)-NTK \(\Theta^{(L)}_\AA\) is invariant under the action of \(\OO(d)\) and thus zonal.
\end{proposition}

The operators lie in the invariant ring \(\AA\in\LD_d(C(\mathbb{S}^{d-1}))^{\OO(d)}\) is called 
\textbf{(rotationally) symmetric} operators. Clearly, the identity operator \(\EE\), 
which is also the identity element of the ring \(\LD_d(C(\mathbb{S}^{d-1}))\), is invariant under 
any matrix in \(\mathrm{GL}(d;\R)\) and thus symmetric. Moreover, the radial derivative operator \(\RR\) 
and the Laplace operator \(\LL\) are also symmetric:
\begin{equation}\label{E:RR_LL}
	\RR = \sum_{k=1}^{d} z_k \partial_{z_k} = z_1 \frac{\partial}{\partial z_1} + \cdots 
	+ z_d \frac{\partial}{\partial z_d}\quad\mbox{and} \quad \LL = \sum_{k=1}^{d} \partial_{z_k}^2 = 
	\frac{\partial^2}{\partial z_1^{2}} + \cdots +  \frac{\partial^2}{\partial z_d^{2}}. 
\end{equation}
Indeed, we can rewrite the operators as \(\RR=z^\top\partial\) and \(\LL=\partial^\top\partial\). 
Then, for any \(g\in\OO(d)\), we have 
\begin{align*}
	g\cdot\RR & = (g\cdot z)^\top (g\cdot \partial) = (g^{-1} z)^\top (g^{-1} \partial) = z^\top\partial = \RR,\\
	g\cdot\LL & = (g\cdot \partial)^\top (g\cdot \partial) = (g^{-1} \partial)^\top (g^{-1} \partial) 
		= \partial^\top\partial = \LL. 
\end{align*}
Other examples include Helmholtz and biharmonic operators, which actually 
lie in the ring generated by \(\RR\) and \(\LL\). 
If the coefficient ring \(\mathfrak{R}\) is the constant function \(\R\) or smooth function \(C^{\infty}(\R^d)\), then the structure 
of the invariant ring is well-understood by first and second fundamental theorems of invariant theory.
However, it is generally challenging to compute the invariant ring \cite{Neusel:2007invariant}. 

For zonal kernels on the unit sphere $\mathbb{S}^{d-1}$ with dimension \(d\geq 2\), their induced integral operators share a 
common set of eigenfunctions given by spherical harmonics, which form an 
orthonormal basis of \(L^2(\mathbb{S}^{d-1})\). Moreover, 
for any zonal kernel \(K(x,y)=K(x^\top y):=K(t)\) such that \(\int_{-1}^{1} |K(t)|(1-t^2)^{(d-3)/2}~dt <\infty\), 
by the \emph{Funk-Hecke formula} \cite[Theorem 2.22]{Atkinson:2012}, we have for any 
spherical harmonics \(Y_n\) of degree \(n\in\N\), 
\begin{align*}
	T_K[Y_n](x) = \int_{\mathbb{S}^{d-1}} K(x,y)Y_n(y) ~ {\rm d}S^{d-1}(y) 
	= \int_{\mathbb{S}^{d-1}} K(x^\top y)Y_n(y) ~ {\rm d}S^{d-1}(y)
	= \lambda_n(K) Y_n(x),
\end{align*}
such that the eigenvalue \(\lambda_n(K)\) is given by 
\begin{equation}\label{E:F-H_coeff}
	\lambda_n(K) = \omega_{d-2} \int_{-1}^{1} K(t)P_{n,d}(t)(1-t^2)^{\frac{d-3}{2}}~{\rm d}t, 
\end{equation}
with \(\omega_{d-2}\) the area of the sphere \(\mathbb{S}^{d-2}\) and 
\(P_{n,d}\) the Lengendre polynomial of degree \(n\) in \(d\) dimension.
For the symmetric operators \(\EE,\RR\) and \(\LL\), the corresponding operator-fixed DNTKs are zonal 
by Proposition \ref{P:zonal}. 
Using the zonal property, we can further compute these kernels explicitly for SNNs 
with RePU activation and zero-bias initialization \cite{Geifman:2020Similarity} (i.e., degenerate Gaussian) as in the 
NTK case. First, we review the \emph{arc-cosine kernel} \cite{Cho:2009Kernel} 
arising from the spectral analysis of the NTK, which is also the basic constituent of these zonal 
operator-fixed DNTKs. Let \(\sigma_s = (x)_+^s\) be the RePU function for some \(s\in\N\). 
Then the \(s\)-th arc-cosine kernel \(k_s\) (up to a constant) on the Euclidean space \(\R^d\) is given by 
\begin{equation}\label{E:arc-cos}
	k_s(x,y) = \E_{u\sim\NN(0,I_d)}\left[ \sigma_s(u^\top x) \sigma_s(u^\top y) \right]
	= \int_{\R^d} \frac{e^{-\frac{1}{2}u^\top u}}{(\sqrt{2\pi})^{d}} 
			(u^\top x)_+^s (u^\top y)_+^s ~ {\rm d}u.
\end{equation}
Let \(\theta = \arccos\big(\ol{x}^\top \ol{y}\big)\) with \(\ol{x}=x/|x|\). 
Then, by \cite{Cho:2009Kernel}, the kernel \(k_s\) is zonal such that 
\begin{equation}
	k_s(x,y) = \frac{1}{2\pi}|x|^s|y|^s J_s(\theta), \qquad \text{ with } \quad
	J_s(\theta) = (\sin\theta)^{2s+1}
				\left(\frac{-1}{\sin\theta}\frac{\partial}{\partial \theta}\right)^s
				\left(\frac{\pi-\theta}{\sin\theta}\right).
\end{equation}
For the zonal operator-fixed DNTK induced by \(\EE,\RR\) and \(\LL\), we can also compute their expressions 
and  obtain the eigenvalue decay of the integral operators, from which we can further compare their RKHSs.

\begin{corollary}\label{C:formula-sym} 
	Let \(\sigma=\sigma_s\) be the RePU function and the bias terms initialized to degenerate Gaussian 
	for SNNs with input dimension \(d\geq 2\).
	Let \(K^\EE_s, K^\RR_s\) and \(K^\LL_s\) (with \(s\geq 1,2,3\) for \(\EE\), $\RR$ and $\LL$) be the corresponding 
	operator-fixed DNTKs given by \eqref{E:2-DNTK} restricted on the unit sphere \(\mathbb{S}^{d-1}\). 
	\begin{itemize}
		\item[(i)] Let \((s)_k = s(s-1)\cdots(s-k+1)\) be the falling factorial and \(m_{d,s} = m_{s}\left(\chi^2_{d}\right)\) 
		the \(s\)-th order noncentral moment of the $\chi$-squared distribution with degree \(d\).
		With \(t:=x^\top y\), we have 
		\begin{align*}
			K^\EE_s(x,y) & = k_s(t) + \gamma^2 + s^2 ~ t k_{s-1}(t) + \gamma^2 s^2 k_{s-1}(t), \\
			K^\RR_s(x,y) & = s^2 k_s(t) + s^4 ~ t k_{s-1}(t) + \gamma^2 (s)_2^2 k_{s-1}(t), \\
			K^\LL_s(x,y) & = \frac{m_{d,s}+4(s-1)m_{d,s-1}}{m_{d,s-2}}(s)_2^2 k_{s-2}(t) 
			+ \frac{m_{d,s-1}}{m_{d,s-3}}(s)_3^2 ~ t k_{s-3}(t) 
			+ \gamma^2 \frac{m_{d,s-1}}{m_{d,s-3}} (s)_3^2 k_{s-3}(t). 
		\end{align*}

		\item[(ii)] Let \(\lambda_k(K^\EE_s)\) with \(k\in\N\) be the \(k\)-th eigenvalue of \(K^\EE_s\) 
		defined by \eqref{E:F-H_coeff}, similarly for \(K^\RR_s\) and \(K^\LL_s\).
		Then for \(\gamma>0\) the eigenvalues satisfy:
		\begin{itemize}
			\item[$\bullet$] \(\lambda_k(K^\EE_s)>0\) for all \(k\geq 0\) and 
				\(\lambda_k(K^\EE_s) \sim k^{-d-2(s-1)}\) when \(k\to \infty\);
			\item[$\bullet$] \(\lambda_k(K^\RR_s)>0\) for all \(k\geq 0\) and 
				\(\lambda_k(K^\RR_s) \sim k^{-d-2(s-1)}\) when \(k\to \infty\);
			\item[$\bullet$] \(\lambda_k(K^\LL_s)>0\) for all \(k\geq 0\) and 
				\(\lambda_k(K^\LL_s) \sim k^{-d-2(s-3)}\) when \(k\to \infty\).
		\end{itemize} \vskip 0.4em

		\item[(iii)] Let \(\HH_{K^\EE_s}, \HH_{K^\RR_s}\) and \(\HH_{K^\LL_s}\) be the RKHSs of 
			\(K^\EE_s,K^\RR_s\) and \(K^\LL_s\) respectively. Then, for \(\gamma>0\) and \(s\geq 3\), we have 
			the relation \( \HH_{K^\EE_s} = \HH_{K^\RR_s} \subset \HH_{K^\LL_s} \).
	\end{itemize}	
\end{corollary}

\begin{remark}
	If \(\gamma=0\), i.e., without the bias term in the network, then we have \(K^\RR_s = s^2 K^\EE_s\) by part (i),
	or the kernel \(K^\RR_s\) is a scaling of the kernel \(K^\EE_s\). 
	Moreover, without the bias term, \(\lambda_k(K^\EE_s) = \lambda_k(K^\RR_s) = 0\) for \(k>s\) and \(k\equiv s \!\mod 2\)
	while \(\lambda_k(K^\LL_s) = 0\) for \(k>s-2\) and \(k\equiv s-2 \!\mod 2\) in part (ii), 
	which will make their RKHSs smaller. 
	These observations imply that the presence of the bias term may improve the expressivity of neural networks. 
	By part (iii), the existence of differential operators does not restrict the expressivity of networks. 
	Instead, the Laplace operator \(\LL\) reduces the smoothness of the RePU function and thereby enlarging the RKHS
	as also in Corollary \ref{C:sum-partial}.
\end{remark}

For the \(\partial_i\)-NTK \(\Theta^{(2)}_{\partial_i}:\R^d\times\R^d\to\R\) induced by the non-symmetric operator 
\(\partial_i\), we can construct a zonal kernel that can be computed explicitly by summing the kernels 
\(\Theta^{(2)}_{\partial_i}\) for all \(i\in[d]\), 
from which we can embed the RKHS of \(\Theta^{(2)}_{\partial_i}\) to some known one.

\begin{corollary}\label{C:sum-partial} 
	Let \(\sigma=\sigma_s\) be the RePU function and the bias terms initialized to degenerate Gaussian 
	for SNNs with input dimension \(d\geq 2\). Let \(K^{\partial_i}_s\) with \(s\geq 2\) and \(i\in[d]\) be the corresponding 
	operator-fixed DNTKs of \(\partial_i\) given by \eqref{E:deep-DNTK} restricted on the unit sphere \(\mathbb{S}^{d-1}\). 
	Then the kernel \(K^{\partial}_s = K^{\partial_1}_s+\cdots+K^{\partial_d}_s\) is zonal such that 
	\[ K^{\partial}_s(x,y) = \left(\frac{m_{d,s}}{m_{d,s-1}}+2s-1\right) s^2 k_{s-1}(t) 
		+ \frac{m_{d,s-1}}{m_{d,s-2}} (s)_2^2 ~ t k_{s-2}(t)
		+ \gamma^2 \frac{m_{d,s-1}}{m_{d,s-2}} (s)_2^2 k_{s-2}(t) \quad (t=x^\top y). \]
	For \(\gamma>0\), the eigenvalues of \(K^\partial_s\) satisfy \(\lambda_k(K^\partial_s)>0\) for all \(k\geq 0\)  
		and \(\lambda_k(K^\partial_s) \sim k^{-d-2(s-2)}\) when \(k\to \infty\), 
		thus \( \HH_{K^\EE_s} = \HH_{K^\RR_s} \subset \HH_{K^\partial_s} \subset \HH_{K^\LL_s} \) for \(s\geq 3\). 
	Let \(P_{ij}\) be the permutation matrix by commuting the \(i\)-th and \(j\)-th row for \(i,j\in[d]\) with \(P_{ii}=I_d\). 
	Then, for any \(k\in[d]\) we have \(\HH_{K^\partial_s} = \bigoplus_{j=1}^d P_{kj}\cdot \HH_{K^{\partial_k}_s}\).
\end{corollary}

\section{Conclusion} \label{S:Conclusion}

In this work, we have investigated the symmetry of the DNTK under the orthogonal group for both shallow and deep neural networks, 
where the action on differential operators stems from the invariant theory based on the identification between 
operators and their symbols.
Using the symmetry, we can compute the operator-fixed DNTKs induced by symmetric operators for SNNs 
with RePU activation and zero-bias initialization. Moreover, we can exploit the group action to construct 
zonal kernels from the DNTKs induced by the non-symmetric operators and then derive the property the DNTKs therefrom. 

There are several promising directions for future research. 
First, while we have provided the symmetry of all components of the DNTK for SNNs, the property should be 
also applicable to DNNs, i.e., the symmetry of \(\Theta^{(L)}_{W^{(\ell)}}\) and \(\Theta^{(L)}_{b^{(\ell)}}\) 
for \(\ell\in[L]\) induced by the weight and bias parameters at \(\ell\)-th layer. 
To this end, we must first derive their expressions and then prove the invariance, analogous to the shallow case. Such an investigation will elucidate the structure of the DNTK for DNNs and may also reveal the role of network depth during training. 
Second, the symmetry makes some DNTKs computable like in Corollaries 
\ref{C:formula-sym} and \ref{C:sum-partial}. We expect the analysis of the property of DNTKs 
for non-symmetric operators using the symmetry. 
Finally, we only consider the DNTKs with single fixed operator on the unit sphere in Section \ref{S:Conseq}, while there are at least two or more operators and more complex domains in a PDE problem, which makes the setting less impractical in scientific tasks. Moreover, it is the kernel matrix and its spectral property, e.g., the extreme eigenvalues and condition number, that controls the training dynamics. Therefore, a precise spectral analysis of the DNTK and the induced matrices are urgent to link the theoretical findings and practical performance of PINNs.
 
\bibliographystyle{abbrv}
\bibliography{reference}

\appendix

\section{Proofs in Section \ref{S:Symmetry}} \label{S-A:Symmetry}

\subsection{Proof of Theorem \ref{T:Sym-SNN}}

To deal with the third term in the expression \eqref{E:2-DNTK} of \(\Theta^{(2)}_{\infty}\), we first compute 
the composition of a general operator with the coordinate function.

\begin{lemma}\label{L:AA-z_k}
Let \(\AA\in \LD_d(\mathfrak{R})\) and \(z_k\) the coordinate function for \(k\in [d]\). Then there holds
$$\AA[\sigma(z)z_k] = \left(z_k\AA +\AA_{(k)}\right) [\sigma(z)],$$ 
where \(\AA_{(k)}\in \LD_d(\mathfrak{R})\) is the operator whose symbol is obtained from \(\AA\) by taking derivative on \(\mathrm{sbl}(\AA)\) with respect to \(\xi_k\):
\[ 
\mathrm{sbl}(\AA\circ z_k) = \mathrm{sbl}(z_k\AA) + \mathrm{sbl}\left(\AA_{(k)}\right) 
	= \left(z_k + \frac{\partial}{\partial \xi_k}\right)\mathrm{sbl}(\AA).   
\]
\end{lemma}
\begin{proof}
First  consider the monomial case \(\AA= \partial^\alpha\) with \(\alpha = (\alpha_1,\ldots,\alpha_k,\ldots,\alpha_d)\). 
If \(\alpha_k=0\), then the assertion holds by the equalities
	\[ \partial^\alpha[\sigma(z)z_k] = z_k\partial^\alpha\sigma(z) \quad \text{ and } \quad
	\frac{\partial}{\partial \xi_k}\xi^\alpha = 0. \]
	If \(\alpha_k>0\), let \(e_k = (0,\ldots,1,\ldots,0)\) be the \(k\)-th Cartesian coordinate vector. Then 
	\[ \frac{\partial}{\partial \xi_k}\xi^\alpha = \frac{\partial}{\partial \xi_k}
	 \xi_1^{\alpha_1}\cdots \xi_k^{\alpha_k}\cdots \xi_d^{\alpha_d} 
	 = \alpha_k \xi_1^{\alpha_1}\cdots \xi_k^{\alpha_k - 1}\cdots \xi_d^{\alpha_d}
	 = \alpha_k\xi^{\alpha - e_k}.  \]
	Meanwhile, by the high-order Leibniz rule, we have
	\begin{align*}
		\partial^\alpha[\sigma(z)z_k] 
		& = \left(\prod_{i\neq k}\partial_i^{\alpha_i}\right)\partial_k^{\alpha_k}[\sigma(z)z_k]
		  = \left(\prod_{i\neq k}\partial_i^{\alpha_i}\right)\sum_{j=0}^{\alpha_k}\binom{\alpha_k}{j}
			\partial_k^{\alpha_k-j}\sigma(z)\partial_k^j z_k \\
		& = \left(\prod_{i\neq k}\partial_i^{\alpha_i}\right) \left(z_k\partial_k^{\alpha_k}\sigma(z) 
			+ \alpha_k\partial_k^{\alpha_k - 1}\sigma(z)\right)\\
		& = z_k\left(\prod_{i\neq k}\partial_i^{\alpha_i}\right)\partial_k^{\alpha_k}\sigma(z) + 
			\alpha_k\left(\prod_{i\neq k}\partial_i^{\alpha_i}\right)\partial_k^{\alpha_k - 1}\sigma(z) \\
		& = z_k\partial^\alpha[\sigma(z)] + \partial^{\alpha - e_k}[\sigma(z)].
	\end{align*}
	Therefore, for any multi-index \(\alpha\), the monomial derivatives satisfy the conclusion.
	Finally, for a general operator \(\AA = \sum_{\alpha\in I_\AA}a_\alpha \partial^\alpha\), we have
	\begin{align*}
		\mathrm{sbl}(\AA\circ z_k) 
	& = \mathrm{sbl}\left(\sum_{\alpha\in I_\AA}a_\alpha \partial^\alpha\circ z_k\right)
	  = \sum_{\alpha\in I_\AA} a_\alpha \mathrm{sbl}(\partial^\alpha\circ z_k) 
	  =	\sum_{\alpha\in I_\AA} a_\alpha \left(z_k + \frac{\partial}{\partial \xi_k}\right)\xi^\alpha \\
	& = \sum_{\alpha\in I_\AA} z_ka_\alpha \xi^\alpha + 
		\sum_{\alpha\in I_\AA} a_\alpha\frac{\partial}{\partial \xi_k} \xi^\alpha
	  = \left(z_k + \frac{\partial}{\partial \xi_k}\right)\sum_{\alpha\in I_\AA} a_\alpha\xi^\alpha
	  = \left(z_k + \frac{\partial}{\partial \xi_k}\right)\mathrm{sbl}(\AA).
	\end{align*}
\end{proof}

Next we compute the inner product between the gradient of monomials 
under the transformation by an orthogonal matrix.

\begin{lemma}\label{L:Inner_Gra-g}
	Let \(g\in \OO(d)\) and \(g_k\) the \(k\)-th row of \(g\) for \(k\in[d]\).
	Let \(u=[u_1,\ldots,u_d]^\top\) represent the variable in the expectation term of \(\Theta^{(2)}_{\infty}\) 
	and \((gu)_k = g_k u\) the \(k\)-th component of the vector \(gu\). 
	Then for any two multi-indices $\alpha$ and $\beta$, we have 
	\begin{equation*}
		\llan \nabla_u (gu)^{\alpha}, \nabla_u (gu)^{\beta} \rran
		= \llan \nabla_{gu} (gu)^{\alpha}, \nabla_{gu} (gu)^{\beta} \rran,
	\end{equation*}
	where \(\nabla_{gu} (gu)^{\alpha}\) is the vector such that for \(k\in[d]\),  
	\begin{equation*}
		[\nabla_{gu} (gu)^{\alpha}]_k = \frac{\partial}{\partial (gu)_k}(gu)^{\alpha} = \begin{cases}
		\alpha_k (g u)^{\alpha_k - 1}_k \prod_{j\neq k}(g u)^{\alpha_j}_j, & \text{if } \alpha_k>0,\\
		0, & \text{if } \alpha_k=0.
	\end{cases}
	\end{equation*}
That is, there holds \(\llan \nabla_u (gu)^{\alpha}, \nabla_u (gu)^{\beta} \rran
		= \llan \nabla_{w} w^{\alpha}, \nabla_{w} w^{\beta} \rran\) under the transformation \(w=gu\).
\end{lemma}

\begin{proof}
	Let \(\ol{g}_k = (g^\top)_k^\top\) be the \(k\)-th column of \(g\in \OO(d)\).
	For the term \((gu)^\alpha\) and \(k\in[d]\), we have
	\begin{align*}
		\frac{\partial}{\partial u_k}(gu)^\alpha 
		& = \frac{\partial}{\partial u_k}(gu)^{\alpha_1}_1\cdots (gu)^{\alpha_d}_d
		 = \frac{\partial}{\partial u_k}(g_1u)^{\alpha_1}\cdots (g_du)^{\alpha_d} \\
		& = \sum_{\alpha_i>0} \alpha_i g_{ik}(g u)^{\alpha_i - 1}_i \prod_{j\neq i}(g u)^{\alpha_j}_j 
		 = \llan  \ol{g}_k , \nabla_{gu} (gu)^{\alpha} \rran.
	\end{align*}
	Note that \(\sum_{k=1}^{d} \ol{g}_k \ol{g}_k^\top = I_d\) for \(g\in \OO(d)\). Therefore, we have
	\begin{align*}
		\llan \nabla_u (gu)^{\alpha}, \nabla_u (gu)^{\beta} \rran 
		& =\sum_{k=1}^{d} \frac{\partial}{\partial u_k}(gu)^\alpha \frac{\partial}{\partial u_k}(gu)^\beta 
		 = \sum_{k=1}^{d}\left( \nabla_{gu} (gu)^{\alpha}\right)^\top \ol{g}_k \cdot 
		  \ol{g}_k^\top \nabla_{gu} (gu)^{\beta} \\
		& = \left( \nabla_{gu} (gu)^{\alpha}\right)^\top \left(\sum_{k=1}^{d} g_k g_k^\top\right) \nabla_{gu} (gu)^{\beta} 
		 = \llan \nabla_{gu} (gu)^{\alpha}, \nabla_{gu} (gu)^{\beta} \rran.
	\end{align*} 
\end{proof}

\begin{proof}[Proof of Theorem \ref{T:Sym-SNN}]
It suffices to prove the symmetry of all the summands of the DNTK for two-layer NNs. 
Note that \(\Theta^{(2)}_{\infty}\) can be split into
\begin{align*}
\Theta^{(2)}_{\infty}\left( x^\AA,y^\BB \right) 
	= \Theta^{(2)}_{W^{(2)}}(x^\AA,y^\BB) + \gamma^2 \Theta^{(2)}_{b^{(2)}}(x^\AA,y^\BB) 
		 + \Theta^{(2)}_{W^{(1)}}(x^\AA,y^\BB) + \gamma^2 \Theta^{(2)}_{b^{(1)}}(x^\AA,y^\BB).
\end{align*}
with the components given by
\begin{align*}
	\Theta^{(2)}_{W^{(2)}}(x^\AA,y^\BB) & = \E_{u,v}\left[ \mathcal{A}\sigma\left(u^\top x + \gamma v\right)
	\mathcal{B}\sigma\left(u^\top y + \gamma v\right) \right],  \quad
	\Theta^{(2)}_{b^{(2)}}(x^\AA,y^\BB)   = a(x)b(y), \\
	\Theta^{(2)}_{W^{(1)}}(x^\AA,y^\BB) & = \E_{u,v}
	\left[ \sum_{k=1}^{d}\mathcal{A}\left(\sigma'\left(u^\top x + \gamma v\right)x_k\right)
	\mathcal{B}\left(\sigma'\left(u^\top y + \gamma v\right)y_k\right) \right], \\
	\Theta^{(2)}_{b^{(1)}}(x^\AA,y^\BB) & =  \E_{u,v}
	\left[ \mathcal{A}\sigma'\left(u^\top x + \gamma v\right)
	\mathcal{B}\sigma'\left(u^\top y + \gamma v\right) \right].
\end{align*}
Given \(\AA = \sum_{\alpha \in I_\AA} a_\alpha \partial^\alpha\) 
with \(a_\alpha\in \mathfrak{R}\) and \(g\in \mathrm{GL}(d;\R)\), we have 
\begin{equation*}
\begin{aligned}
	\AA[\sigma(u^\top z + \gamma v)] & = 
	\sum_{\alpha \in I_\AA} a_\alpha(z)\sigma^{(|\alpha|)}(u^\top z + \gamma v) u^\alpha, \\
	(g^{-1}\cdot \AA) [\sigma(u^\top z + \gamma v)] & = 
	\left(\sum_{\alpha \in I_\AA} (a_\alpha\circ g) (g\partial)^\alpha\right) \sigma(u^\top z + \gamma v) 
	= \sum_{\alpha \in I_\AA} a_\alpha(gz)\sigma^{(|\alpha|)}\left(u^\top z + \gamma v\right) (gu)^\alpha. 	
\end{aligned}
\end{equation*}
Therefore, for any \(x\in \R^d\) and \(g\in\OO(d)\), since \(g^{-1}=g^\top\), we have 
\begin{align*}
	&(g^{-1}\cdot \AA)\sigma\left(u^\top (g^{-1}x) + \gamma v\right) 
	 =  (g^{-1}\cdot \AA)[\sigma(u^\top z + \gamma v)]~ \Big|_{z=g^{-1}x} \\
	=& \sum_{\alpha \in I_\AA} a_\alpha(g g^{-1}x)\sigma^{(|\alpha|)}\left(u^\top (g^{-1}x) + \gamma v\right) (gu)^\alpha \\
    = & \sum_{\alpha \in I_\AA} a_\alpha(x)\sigma^{(|\alpha|)}
		\left((gu)^\top x + \gamma v\right) (gu)^\alpha.
\end{align*}
First, consider the terms \(\Theta^{(2)}_{W^{(2)}}\) and \(\Theta^{(2)}_{b^{(2)}}\) which have the same form.
Let \(u\sim \NN(0,I_d)\) be the standard Gaussian vector. For any orthogonal matrix \(g\in\OO(d)\), the 
vector \(w=gu\) is still standard Gaussian. Indeed, let \(f(u) = \frac{1}{\sqrt{(2\pi)^d}}e^{-\frac{u^\top u}{2}} \) be the 
density function of \(u\). Since \(g^\top g=I\) for \(g\in\OO(d)\), the density function \(f(w)=f(gu)\) is invariant 
under the transformation \(g\): 
\[ f(w) = f(gu) = \frac{1}{\sqrt{(2\pi)^d}}e^{-\frac{(gu)^\top (gu)}{2}} 
= \frac{1}{\sqrt{(2\pi)^d}}e^{-\frac{u^\top u}{2}} = f(u). \]
Let \(\Theta^{(2)}_{W^{(2)}}(x^\AA,y^\BB;v):= \E_{u}\left[ \mathcal{A}\sigma\left(u^\top x + \gamma v\right)
	\mathcal{B}\sigma\left(u^\top y + \gamma v\right) \right]\). Then, for all \(v\in\R\), we have
\begin{align*}
	&g\cdot \Theta^{(2)}_{W^{(2)}} \left(x^\AA,y^\BB;v\right) =
	\Theta^{(2)}_{W^{(2)}} \left(g^{-1}\cdot x^\AA,g^{-1}\cdot y^\BB;v\right)  \\
	 =& \E_{u}\left[ (g^{-1}\cdot \AA)\sigma\left(u^\top (g^{-1}x) + \gamma v\right)
	(g^{-1}\cdot \BB)\sigma\left(u^\top (g^{-1}y) + \gamma v\right) \right]\\
	=& \int_{\R^d} f(u) \d u \cdot
		\sum_{\alpha \in I_\AA} a_\alpha(x)\sigma^{(|\alpha|)}
		\left((gu)^\top x + \gamma v\right) (gu)^\alpha \cdot 
		\sum_{\beta \in I_\BB} b_\beta(y)\sigma^{(|\beta|)}
		\left((gu)^\top y + \gamma v\right) (gu)^\beta \\
	=& \int_{\R^d} f(w)\d w\cdot
		\sum_{\alpha \in I_\AA} a_\alpha(x)\sigma^{(|\alpha|)}
		\left(w^\top x + \gamma v\right) w^\alpha
		\sum_{\beta \in I_\BB} b_\beta(y)\sigma^{(|\beta|)}
		\left(w^\top y + \gamma v\right) w^\beta \\
	 =& \E_{w}\left[ \AA\sigma\left(w^\top x + \gamma v\right)\BB\sigma\left(w^\top y + \gamma v\right) \right]
	  = \Theta^{(2)}_{W^{(2)}}\left(x^\AA,y^\BB;v\right), 
\end{align*}
using the coordinate transformation \(w=gu\) and 
\(f(gu)=f(u),\int \d(gu) = \int |\mathrm{det}(g)|\d u = \int \d u\). 
Therefore, for any \(x^\AA,y^\BB\in \PLDR_d(\mathfrak{R})\), we have 
\[ g\cdot \Theta^{(2)}_{W^{(2)}} \left(x^\AA,y^\BB\right) 
	= \E_v\left[ \Theta^{(2)}_{W^{(2)}} \left(g^{-1}\cdot x^\AA,g^{-1}\cdot y^\BB;v\right) \right]
	= \E_v\left[ \Theta^{(2)}_{W^{(2)}} \left(x^\AA,y^\BB;v\right) \right] = \Theta^{(2)}_{W^{(2)}} \left(x^\AA,y^\BB\right). \]
This shows \(g\cdot \Theta^{(2)}_{W^{(2)}} = \Theta^{(2)}_{W^{(2)}}\) for \(g\in\OO(d)\) and the invariance of 
\(\Theta^{(2)}_{b^{(1)}}\) follows similarly. 
Next, consider the kernel \(\Theta^{(2)}_{b^{(2)}}(x^\AA,y^\BB) = a(x)b(y)\) with $a $ and $b$ the zeroth terms of  $\AA$ and $\BB$, respectively. 
Note that the zeroth terms of the operators \(g^{-1}\cdot\AA,~g^{-1}\cdot\BB\) are \(a\circ g,~b\circ g\) for 
any \(g\in \mathrm{GL}(d;\R)\). Thus for any invertible matrix \(g\) we have 
\begin{align*}
	g\cdot \Theta^{(2)}_{b^{(2)}}(x^\AA,y^\BB) = \Theta^{(2)}_{b^{(2)}}(g^{-1}\cdot x^\AA, g^{-1}\cdot y^\BB) 
	=  a(gz)\mid_{z=g^{-1}x} b(gz)\mid_{z=g^{-1}y} = \Theta^{(2)}_{b^{(2)}}(x^\AA,y^\BB).
\end{align*}
That is, \(g\cdot \Theta^{(2)}_{b^{(2)}}=\Theta^{(2)}_{b^{(2)}}\) for any \(g\in \mathrm{GL}(d;\R)\) and 
\(\Theta^{(2)}_{b^{(2)}}\) is invariant under the orthogonal group \(\OO(d)\). 
Since \(\Sigma^{(2)}_{\infty} = \Theta^{(2)}_{W^{(2)}} + \Theta^{(2)}_{b^{(2)}} \), 
the covariance function \(\Sigma^{(2)}_{\infty}\) is also invariant under \(\OO(d)\).
Finally, consider the term \(\Theta^{(2)}_{W^{(1)}}\). By Lemma \ref{L:AA-z_k}, we have
\begin{align*}
	\Theta^{(2)}_{W^{(1)}}(x^\AA, y^\BB) =& \E_{u,v}\left[ \sum_{k=1}^{d} 
	\left(x_k\AA + \AA_{(k)} \right) \sigma'(u^\top x + \gamma v)
	\cdot \left(y_k\BB + \BB_{(k)}\right)\sigma'(u^\top y + \gamma v) \right] \\
	=& \E_{u,v}\left[ \sum_{k=1}^{d} 
	x_k y_k\AA \sigma'(u^\top x + \gamma v) \BB\sigma'(u^\top y + \gamma v)
	+ \sum_{k=1}^{d} \AA_{(k)} \sigma'(u^\top x + \gamma v) \BB_{(k)}\sigma'(u^\top y + \gamma v) \right. \\
	& + \left. \sum_{k=1}^{d} x_k \AA \sigma'(u^\top x + \gamma v)\BB_{(k)}\sigma'(u^\top y + \gamma v)
		+ y_k \BB \sigma'(u^\top y + \gamma v)\AA_{(k)}\sigma'(u^\top x + \gamma v) \right] \\
	 =&: x^\top y\cdot \Theta^{(2)}_{b^{(1)}}(x^\AA, y^\BB) + K_\mathrm{I}(x^\AA, y^\BB) + K_\mathrm{II}(x^\AA, y^\BB), 
\end{align*}
with the kernels \(K_\mathrm{I}\) and \(K_\mathrm{II}\) given respectively by (with \(\dot{\sigma}=\sigma'\))
\begin{align*}
	K_\mathrm{I}(x^\AA, y^\BB) & = \E_{u,v}\left[ 
		\sum_{k=1}^{d} \AA_{(k)} \dot{\sigma}(u^\top x + \gamma v) \BB_{(k)}\dot{\sigma}(u^\top y + \gamma v) \right], \\
	K_\mathrm{II}(x^\AA, y^\BB) & = \E_{u,v}\left[
		\sum_{k=1}^{d} x_k \AA \dot{\sigma}(u^\top x + \gamma v)\BB_{(k)}\dot{\sigma}(u^\top y + \gamma v)
		+ y_k \BB \dot{\sigma}(u^\top y + \gamma v)\AA_{(k)}\dot{\sigma}(u^\top x + \gamma v)\right].
\end{align*}
Then, for any \(g\in \OO(d)\), we have 
\begin{align*}
	g\cdot \Theta^{(2)}_{W^{(1)}} = \Theta^{(2)}_{W^{(1)}}(g^{-1}\cdot x^\AA, g^{-1}\cdot y^\BB) = 
	({g^{-1}x})^\top (g^{-1}y)\cdot \Theta^{(2)}_{b^{(1)}}(g^{-1}\cdot x^\AA, g^{-1}\cdot y^\BB)
	+ g\cdot ( K_\mathrm{I} + K_\mathrm{II} )(x^\AA, y^\BB).
\end{align*}
Note that 
$$({g^{-1}x})^\top (g^{-1}y) = x^\top g^{-\top}g^{-1}y = x^\top y\quad \mbox{and}\quad g\cdot \Theta^{(2)}_{b^{(1)}} = \Theta^{(2)}_{b^{(1)}}.$$ 
Thus it suffices to prove the invariance of the kernels \(K_\mathrm{I}\) and \(K_\mathrm{II}\).
Let \(I_\AA^+\) be the indices \(\alpha\) in \(I_\AA\) with \(|\alpha|>0\) for \(\AA\in \LD_d(\mathfrak{R})\). 
Then, given \(g\in \mathrm{GL}(d;\R)\) and \(k\in [d]\), by Lemma \ref{L:AA-z_k}, we have 
\begin{align*}
	\mathrm{sbl}\left((g^{-1}\cdot \AA)_{(k)}\right) 
	= \frac{\partial}{\partial \xi_k} \mathrm{sbl}\left(g^{-1}\cdot \AA\right)
	=  \frac{\partial}{\partial \xi_k}\sum_{\alpha \in I_\AA} (a_\alpha\circ g) (g\xi)^\alpha
	= \sum_{\alpha \in I_\AA^+} (a_\alpha\circ g) \frac{\partial}{\partial \xi_k} (g\xi)^\alpha.
\end{align*}
Thus, there hold
\begin{align*}
	\AA_{(k)}[\dot{\sigma}(u^\top z + \gamma v) ]
	& = \sum_{\alpha \in I_\AA^+} a_\alpha(z)\dot{\sigma}^{(|\alpha|-1)}
		\left(u^\top z + \gamma v\right)\frac{\partial}{\partial u_k} u^\alpha,\\
	(g^{-1}\cdot \AA)_{(k)} [\dot{\sigma}(u^\top z + \gamma v)] 
	& = \sum_{\alpha \in I_\AA^+} a_\alpha(gz)\dot{\sigma}^{(|\alpha|-1)}
		\left(u^\top z + \gamma v\right)\frac{\partial}{\partial u_k} (gu)^\alpha.
\end{align*}
Therefore, for any \(x\in \R^d\) and \(g\in\OO(d)\), we have
\begin{align*}
	(g^{-1}\cdot \AA)_{(k)} \dot{\sigma}\left(u^\top (g^{-1}x) + \gamma v\right)  
	& = (g^{-1}\cdot \AA)_{(k)} [\dot{\sigma}(u^\top z + \gamma v)] \mid_{z=g^{-1}x}  \\
	& = \sum_{\alpha \in I_\AA^+} a_\alpha(x)\dot{\sigma}^{(|\alpha|-1)}
		\left((gu)^\top x + \gamma v\right)\frac{\partial}{\partial u_k} (gu)^\alpha.
\end{align*}
We denote the integrand under the expectation in $K_\mathrm{I}$ and $K_\mathrm{II}$ by 
$\kappa_\mathrm{I}$ and $ \kappa_\mathrm{II}$, respectively. Then, we have 
\begin{align*}
	\kappa_\mathrm{I}(x^\AA, y^\BB) 
	& = \sum_{k=1}^{d} \AA_{(k)} \dot{\sigma}(u^\top x + \gamma v) \BB_{(k)}\dot{\sigma}(u^\top y + \gamma v) \\
	& = \sum_{k=1}^{d} \sum_{\alpha \in I_\AA^+} a_\alpha(x)\dot{\sigma}^{(|\alpha|-1)}
		\left(u^\top x + \gamma v\right)\frac{\partial}{\partial u_k} u^\alpha
		\sum_{\beta \in I_\BB^+} b_\beta(y)\dot{\sigma}^{(|\beta|-1)}
		\left(u^\top y + \gamma v\right)\frac{\partial}{\partial u_k} u^\beta \\
	& = \sum_{\alpha \in I_\AA^+}\sum_{\beta \in I_\BB^+}a_\alpha(x)b_\beta(y)
		\dot{\sigma}^{(|\alpha|-1)}\left(u^\top x + \gamma v\right)
		\dot{\sigma}^{(|\beta|-1)}\left(u^\top y + \gamma v\right)
		\llan \nabla_u u^{\alpha}, \nabla_u u^{\beta} \rran.
\end{align*}
The integrand \(\kappa_\mathrm{I}(x^\AA, y^\BB)\) after the action under \(g\in \OO(d)\) is given by
\begin{align*}
	g\cdot \kappa_\mathrm{I}(x^\AA, y^\BB)
	& = \kappa_\mathrm{I}(g^{-1}\cdot x^\AA, g^{-1}\cdot y^\BB) 
	  = \sum_{k=1}^{d} (g^{-1}\cdot \AA)_{(k)} \dot{\sigma}(u^\top (g^{-1}x) + \gamma v) 
		(g^{-1}\cdot \BB)_{(k)}\dot{\sigma}(u^\top (g^{-1}y) + \gamma v) \\
	& = \sum_{k=1}^{d} \sum_{\alpha \in I_\AA^+} a_\alpha(x)\dot{\sigma}^{(|\alpha|-1)}
		\left((gu)^\top x + \gamma v\right)\frac{\partial}{\partial u_k} (gu)^\alpha
		\sum_{\beta \in I_\BB^+} b_\beta(y)\dot{\sigma}^{(|\beta|-1)}
		\left((gu)^\top y + \gamma v\right)\frac{\partial}{\partial u_k} (gu)^\beta \\
	& = \sum_{\alpha \in I_\AA^+}\sum_{\beta \in I_\BB^+}a_\alpha(x)b_\beta(y)
		\dot{\sigma}^{(|\alpha|-1)}\left((gu)^\top x + \gamma v\right)
		\dot{\sigma}^{(|\beta|-1)}\left((gu)^\top y + \gamma v\right)
		\llan \nabla_u (gu)^{\alpha}, \nabla_u (gu)^{\beta} \rran.
\end{align*}
By Lemma \ref{L:Inner_Gra-g}, with \(w=gu\), we have 
\[ \llan \nabla_u (gu)^{\alpha}, \nabla_u (gu)^{\beta} \rran = 
\llan \nabla_{gu} (gu)^{\alpha}, \nabla_{gu} (gu)^{\beta} \rran 
= \llan \nabla_{w} w^{\alpha}, \nabla_{w} w^{\beta} \rran,  \]
from which we deduce \( \kappa_\mathrm{I}(g^{-1}\cdot x^\AA, g^{-1}\cdot y^\BB; u,v) = 
\kappa_\mathrm{I}(x^\AA, y^\BB; w,v) \) with \(w=gu\) again being a standard Gaussian vector.
Similar to the case \(\Theta^{(2)}_{W^{(2)}}\), by taking expectation with respect to \(u\) and 
using the coordinate transformation \(w=gu\), we obtain 
\begin{align*}
	K_\mathrm{I} \left(g^{-1}\cdot x^\AA,g^{-1}\cdot y^\BB;v\right)  
	= \E_u\left[ \kappa_\mathrm{I}(g^{-1}\cdot x^\AA, g^{-1}\cdot y^\BB; u,v) \right]
	= \E_w\left[ \kappa_\mathrm{I}(x^\AA, y^\BB; w,v) \right] 
	= K_\mathrm{I} \left( x^\AA,y^\BB;v\right). 
\end{align*}
Then we get the invariance of the kernel \(K_\mathrm{I}\). 
For the kernel \(K_\mathrm{II}\), we first compute the summation which appears as 
the summands of the integrand \(\kappa_\mathrm{II}(x^\AA,y^\BB)\):
\begin{align*}
	&\sum_{k=1}^{d}  x_k\AA \dot{\sigma}(u^\top x + \gamma v) \BB_{(k)}\dot{\sigma}(u^\top y + \gamma v) \\
	 =& \sum_{\alpha \in I_\AA} a_\alpha(x)\dot{\sigma}^{(|\alpha|)}\left(u^\top x + \gamma v\right) u^\alpha
		\sum_{\beta \in I_\BB^+} b_\beta(y)\dot{\sigma}^{(|\beta|-1)}\left(u^\top y + \gamma v\right)
		\sum_{k=1}^{d} x_k  \frac{\partial}{\partial u_k} u^\beta \\
	 =& \sum_{\alpha \in I_\AA} a_\alpha(x)\dot{\sigma}^{(|\alpha|)}\left(u^\top x + \gamma v\right) u^\alpha
		\sum_{\beta \in I_\BB^+} b_\beta(y)\dot{\sigma}^{(|\beta|-1)}
		\left(u^\top y + \gamma v\right) \llan x, \nabla_u u^\alpha \rran.
\end{align*}
Let \(\ol{g}_k = (g^\top)_k^\top\) be the \(k\)-th column of \(g\in \OO(d)\).
By Lemma \ref{L:Inner_Gra-g}, after the action under the orthogonal matrix \(g\in \OO(d)\), we arrive at
\begin{align*}
	&\sum_{k=1}^{d}  z_k (g^{-1}\cdot \AA) [ \dot{\sigma}(u^\top z + \gamma v)]\mid_{z=g^{-1}x} 
		(g^{-1}\cdot \BB)_{(k)}[ \dot{\sigma}(u^\top z + \gamma v)]\mid_{z=g^{-1}y} \\
	=& \sum_{k=1}^{d} (g^{-1}x)_k \sum_{\alpha \in I_\AA} a_\alpha(x) \dot{\sigma}^{(|\alpha|)}
		\left((gu)^\top x + \gamma v\right) (gu)^\alpha
		\sum_{\beta \in I_\BB^+} b_\beta(y) \dot{\sigma}^{(|\beta|-1)}
		\left((gu)^\top y + \gamma v\right)\frac{\partial}{\partial u_k} (gu)^\beta \\
	=& \sum_{\alpha \in I_\AA} a_\alpha(x) \dot{\sigma}^{(|\alpha|)}
		\left((gu)^\top x + \gamma v\right) (gu)^\alpha
		\sum_{\beta \in I_\BB^+} b_\beta(y) \dot{\sigma}^{(|\beta|-1)}
		\left((gu)^\top y + \gamma v\right)
		\sum_{k=1}^{d} \llan  (g^{-1}x)_k \ol{g}_k , \nabla_u (gu)^{\beta} \rran \\
	=& \sum_{\alpha \in I_\AA} a_\alpha(x) \dot{\sigma}^{(|\alpha|)}
		\left((gu)^\top x + \gamma v\right) (gu)^\alpha
		\sum_{\beta \in I_\BB^+} b_\beta(y) \dot{\sigma}^{(|\beta|-1)}
		\left((gu)^\top y + \gamma v\right) \llan x, \nabla_{gu} (gu)^{\beta} \rran,
\end{align*}
in view of the following equality
\begin{align*}
	\sum_{k=1}^{d} \llan  (g^{-1}x)_k \ol{g}_k , \nabla_u (gu)^{\beta} \rran 
	= \llan [\ol{g}_1,\ldots,\ol{g}_d]g^{-1}x , \nabla_{gu} (gu)^{\beta} \rran 
	= \llan x, \nabla_{gu} (gu)^{\beta} \rran.
\end{align*}Similarly, for the second summand of \(\kappa_\mathrm{II}(x^\AA,y^\BB)\), we have
\begin{align*}
	&\sum_{k=1}^{d} y_k\BB \dot{\sigma}(u^\top y + \gamma v) \AA_{(k)}\dot{\sigma}(u^\top x + \gamma v) \\
	 =& \sum_{\beta \in I_\BB} b_\beta(y)\dot{\sigma}^{(|\beta|)}\left(u^\top y + \gamma v\right) u^\beta
		\sum_{\alpha \in I_\AA^+} a_\alpha(x)\dot{\sigma}^{(|\alpha|-1)}
		\left(u^\top x + \gamma v\right) \llan y, \nabla_u u^\beta \rran, \\
	&\sum_{k=1}^{d}  z_k (g^{-1}\cdot \BB) [\dot{\sigma}(u^\top z + \gamma v)]\mid_{z=g^{-1}y} 
		(g^{-1}\cdot \AA)_{(k)}[\dot{\sigma}(u^\top z + \gamma v)]\mid_{z=g^{-1}x} \\
	=& \sum_{\beta \in I_\BB} b_\beta(y)\dot{\sigma}^{(|\beta|)}
		\left((gu)^\top y + \gamma v\right) (gu)^\beta
		\sum_{\alpha \in I_\AA} a_\alpha(x)\dot{\sigma}^{(|\alpha|-1)}
		\left((gu)^\top x + \gamma v\right) \llan y, \nabla_{gu} (gu)^{\alpha} \rran.
\end{align*}
Thus, we have the equality \( \kappa_\mathrm{II}(g^{-1}\cdot x^\AA, g^{-1}\cdot y^\BB; u,v) = 
\kappa_\mathrm{II}(x^\AA, y^\BB; w,v) \) with \(w=gu\). 
Then the invariance of the kernel \(K_\mathrm{II}\) follows as in the case \(K_\mathrm{I}\).
\end{proof}

\subsection{Proof of Theorem \ref{T:Sym-DNN}}

The invariance of \(\Sigma^{(L)}_\infty\) and \(\Theta^{(L)}_\infty\) are proved by mathematical induction 
based on the recursive formulas \eqref{E:deep-Cov} and \eqref{E:deep-DNTK} starting at \(L=1\).
Below the notation under tilde denotes the images of their original counterparts 
under the action by \(g^{-1}\) given any generic \(g\in \mathrm{GL}(d;\R)\). 
For instance, let \(\partial=[\partial_1,\ldots,\partial_d]^\top\) and 
\(\AA = \sum_{\alpha \in I_\AA} a_\alpha \partial^\alpha \in \LD_d(\mathfrak{R})\) with \(a_\alpha\in \mathfrak{R}\).
For \(g\in \mathrm{GL}(d;\R)\) with rows \(g_1,\ldots,g_d\), 
\begin{gather*}
	\wt{a}_\alpha = g^{-1}\cdot a_\alpha =  a_\alpha\circ g, \quad 
	\wt{\partial} = g^{-1}\cdot \partial = g\partial = [g_1^\top\partial, \ldots, g_d^\top \partial]^\top, \\
	\wt{\AA} = g^{-1}\cdot \AA = \sum_{\alpha \in I_\AA} (a_\alpha\circ g) (g\partial)^\alpha
	= \sum_{\alpha \in I_\AA} \wt{a}_\alpha \wt{\partial}^\alpha.
\end{gather*}

First we study the general Leibniz rule and Fa\`a di Bruno's formula under the 
differentiation after transformation \(\wt{\partial}\).

\begin{lemma}\label{L:g-Leibniz} 
Let \(g\in \mathrm{GL}(d;\R)\) with rows \(g_1,\ldots,g_d\) and \(\wt{\partial} = g^{-1}\cdot \partial = g\partial\). 
Then, the following high-order Leibniz rule for two functions holds: given \(f,g:\R^d \to \R\), we have
\[ \wt{\partial}^{\alpha}(fg) = \sum_{\beta \leq \alpha}\binom{\alpha}{\beta}\wt{\partial}^{\beta} f ~ \wt{\partial}^{\alpha-\beta} g \quad \Longleftrightarrow \quad 
(g\partial)^{\alpha}(fg) = \sum_{\beta \leq \alpha}\binom{\alpha}{\beta} (g\partial)^{\beta} f ~ (g\partial)^{\alpha-\beta} g. 
\] 
Moreover, the first-order Leibniz rule for multiple functions holds: for any \(f_1,\ldots,f_n:\R^d \to \R\) and any \(k\in[d]\), we have 
	\[ \wt{\partial}_k(f_1\cdots f_n) = \sum_{j=1}^{n} f_1 \cdots f_{j-1} ~ (\wt{\partial}_k f_j) ~
	f_{j+1} \cdots f_n \quad \Longleftrightarrow \quad 
	(g_k\partial)(f_1\cdots f_n) = \sum_{j=1}^{n} \Big(\prod_{i\neq j}f_i\Big) (g_k\partial)f_j .  \]
\end{lemma}

\begin{proof}
	We prove the first assertion by induction on the index norm \(|\alpha|\). When \(|\alpha| = 1\),
	we have \(\wt{\partial}^{\alpha} = \wt{\partial}_i = g_i^\top \partial = \sum_{j=1}^{d}g_{ij}\partial_j\)
	for some \(i\in[d]\). Then,
	\begin{align*}
		\wt{\partial}_i (fg) & = \sum_{j=1}^{d} \left(g_{ij}\partial_j\right)(fg) 
		= \sum_{j=1}^{d} g_{ij}\Big( (\partial_j f) g + f (\partial_j g) \Big)
		=  \sum_{j=1}^{d} g_{ij}(\partial_j f)g + f g_{ij} (\partial_j g) \\
		& = \sum_{j=1}^{d} \Big( (g_{ij}\partial_j)f \Big)g + f\Big( (g_{ij}\partial_j)g \Big)
		= (\wt{\partial}_i f)g + f(\wt{\partial}_i g).
	\end{align*}
	Suppose the conclusion holds for all muti-index with norm \(|\alpha|\). For any 
	index \(\xi\) with norm \(|\alpha|+1\), we can write it as \(\xi = e_i + \alpha\) 
	for some \(i\in[d]\),
	where \(e_i\) is the $i$th Cartesian coordinate vector and 
	\(\alpha\) is some index with norm \(|\alpha|\).
	By induction, we have
	\begin{align*}
		\wt{\partial}^\xi(fg) & =
		(\wt{\partial}_i\wt{\partial}^\alpha)(fg) = \wt{\partial}_i \left( 
		\sum_{\beta \leq \alpha}\binom{\alpha}{\beta}
		\wt{\partial}^{\beta} f ~ \wt{\partial}^{\alpha-\beta} g \right)
		= \sum_{\beta \leq \alpha}\binom{\alpha}{\beta} 
		\wt{\partial}_i \left( \wt{\partial}^{\beta} f ~ \wt{\partial}^{\alpha-\beta} g \right) \\
		& = \sum_{\beta \leq \alpha}\binom{\alpha}{\beta} \left( 
		(\wt{\partial}_i \wt{\partial}^{\beta} f)\wt{\partial}^{\alpha-\beta} g +
		\wt{\partial}^{\beta} f(\wt{\partial}_i \wt{\partial}^{\alpha-\beta} g) \right)
		= \sum_{\beta \leq \xi}\binom{\xi}{\beta}
		\wt{\partial}^{\beta} f ~ \wt{\partial}^{\xi-\beta} g. 
	\end{align*}
	The second assertion holds using the first-order Leibniz rule 
	for each summand of the operator \(\wt{\partial}_k = g_k\partial\).
\end{proof}

For the Faa di Bruno's formula under \(\wt{\partial}\), recall that in equation \eqref{E:Faa-alpha}, 
for any index \(\alpha=(\alpha_1,\ldots,\alpha_d)\), any function \(h:\R^d \to \R\) and any subset \(A\subset S_\alpha := 
\left\{ 1_{1},\ldots,1_{\alpha_1},2_{1},\ldots,2_{\alpha_2},\cdots, d_{1},\ldots,d_{\alpha_d} \right\}\), we have 
\[ \partial^A h  = \frac{\partial^{|A|} h}{\prod_{j \in A} \partial z_j} =  \prod_{j \in A} \partial_j h. \]
Similarly, for the operator \(\wt{\partial} = g\partial\) with \(g\in\mathrm{GL}(d;\R)\), we have 
\[ \wt{\partial}^A h  =  \prod_{j \in A} \wt{\partial}_j h \quad \Longleftrightarrow \quad
 	(g\partial)^A h = \prod_{j \in A} (g_j \partial) h. \]

\begin{lemma}\label{L:g-Faa di} 
	Let \(g\in \mathrm{GL}(d;\R)\) with rows \(g_1,\ldots,g_d\) and \(\wt{\partial} = g^{-1}\cdot \partial = g\partial\). 
	Given functions \(f:\R\to\R\) and \(h:\R^d \to \R\), we have for any index \(\alpha\), 
	\[ \wt{\partial}^\alpha f(h) = 
	\sum_{\pi \in \Pi_\alpha} f^{(|\pi|)}(h) \prod_{A \in \pi} \wt{\partial}^A h
	\quad \Longleftrightarrow \quad  
	(g\partial)^\alpha f(h) = 
	\sum_{\pi \in \Pi_\alpha} f^{(|\pi|)}(h) \prod_{A \in \pi} (g\partial)^A h. \]
\end{lemma}

\begin{proof}
The proof is based on induction on the index norm \(|\alpha|\). 
	For any \(i\in [d]\), we have
	\begin{align*}
		\wt{\partial}_i f(g) = \sum_{j=1}^{d}(g_{ij}\partial_j)f(g) = \sum_{j=1}^{d}g_{ij}f^{(1)}(g)\partial_j g
		= \sum_{j=1}^{d}f^{(1)}(g) (g_{ij}\partial_j g )= f^{(1)}(g)\wt{\partial}_i g,
	\end{align*}
which completes the proof when \(|\alpha|=1\). Suppose that the assertion holds for all indices with norm \(|\alpha|\geq 1\). For any index \(\xi\) with norm \(|\alpha|+1\), we can write it as \(\xi = e_i + \alpha\) for some \(i\in[d]\). Then, by induction and the base case, we have
	\begin{align*}
		\wt{\partial}^\xi f(g) & = \wt{\partial}_i\left(\wt{\partial}^\alpha f(g)\right) 
		= \wt{\partial}_i\left(\sum_{\pi \in \Pi_\alpha} f^{(|\pi|)}(g) 
		\prod_{A \in \pi} \wt{\partial}^A g\right) \\
		& = \sum_{\pi \in \Pi_\alpha}  f^{(|\pi|+1)}(g) \wt{\partial}_i g \prod_{A \in \pi} \wt{\partial}^A g
			+ f^{(|\pi|)}(g) \wt{\partial}_i \left( \prod_{A \in \pi} \wt{\partial}^A g \right) \\
		& = \sum_{\pi \in \Pi_{e_i + \alpha}} f^{(|\pi|)}(g) \prod_{A \in \pi} \wt{\partial}^A g.
	\end{align*}
	The last equality holds since the partition of the set \(S_{e_i + \alpha}\) can be split 
	into two types. The first type is the disjoint union of a partition of \(S_{\alpha}\), 
	which corresponds to an element \(\pi\in \Pi_\alpha\),
	and a partition of \(S_{e_i}\), which must be \(S_{e_i}\) itself since it is a singleton.
	This kind of partition corresponds to the first summand in the second line. The other type 
	partition of \(S_{e_i + \alpha}\) is obtained from a partition \(\pi\) of \(S_{\alpha}\) by inserting 
	\(S_{e_i}\) into one of the blocks in \(\pi\). For each \(\pi\in\Pi_\alpha\), this step will 
	generate \(|\pi|\) partitions of \(S_{e_i + \alpha}\), and corresponds to the second summand 
	in the second line by the Leibniz rule under \(\wt{\partial}\) (Lemma \ref{L:g-Leibniz}).
\end{proof}

\begin{proof}[Proof of Theorem \ref{T:Sym-DNN}]

We first prove the invariance of \(\Sigma^{(L)}_\infty\) \eqref{E:deep-Cov} by mathematical induction. \\
\underline{Step 1. The base case $L=1$.}
In the base case \(L=1\), we write as $$\Sigma^{(1)}_\infty = \Sigma^{(1)}_W + \gamma^2\Sigma^{(1)}_b,
$$ 
with 
\begin{equation*}
\Sigma^{(1)}_b( x^\AA,y^\BB )=a(x)b(y)=\Theta^{(2)}_{b^{(2)}}(x^\AA,y^\BB)\quad\mbox{and}\quad      
\Sigma^{(1)}_W(x^\AA,y^\BB) = \left(a(x)x+D\AA(x)\right)^\top \left(b(y)y+D\BB(y)\right).
\end{equation*}
$ \Sigma^{(1)}_b( x^\AA,y^\BB )$ is invariant by Theorem \ref{T:Sym-SNN}.
Let \(a_0^\AA\) represent the zeroth term and \(a_k^\AA\) for \(k\in[d]\) the coefficients of \(\partial_k\)
of the operator \(\AA\in\LD_d(\mathfrak{R})\). Then, for any \(g\in\mathrm{GL}(d;\R)\), by the formula \eqref{E:G-DO_formula}, 
we have \( a_0^{g^{-1}\cdot\AA} = a_0^{\AA} \circ g \) and 
\[ a_i^{g^{-1}\cdot\AA} = \sum_{j=1}^{d}  \wt{a}_j g_{ji} 
		= \sum_{j=1}^{d} g_{ji} a_j  \circ g, ~ i\in[d] \quad \Longrightarrow \quad
		D(g^{-1}\cdot\AA) = g^{\top} D(\AA)\circ g.  \]
Thus, for any \(g\in\OO(d)\), we have 
\begin{align*}
	\Sigma^{(1)}_W(g^{-1}\cdot x^\AA, g^{-1}\cdot y^\BB) 
	& = \left(a_0^{g^{-1}\cdot\AA}(z)z+D(g^{-1}\cdot\AA)(z)\right) \Big|_{z=g^{-1}\cdot x}^\top 
		\left(a_0^{g^{-1}\cdot\BB}(z)z+D(g^{-1}\cdot\BB)(z)\right) \Big|_{z=g^{-1}\cdot y}^\top \\
	& = \left(a_0^{\AA}(x)(g^{-1}x)+g^{\top}D\AA(x)\right)^\top 
		\left(a_0^{\BB}(y)(g^{-1}y)+g^{\top}D\BB(y)\right) \\
	& = \left(a_0^{\AA}(x)x+D\AA(x)\right)^\top 
		\left(a_0^{\BB}(y)y+D\BB(y)\right) = \Sigma^{(1)}_W(x^\AA, y^\BB).
\end{align*}
This completes the proof for \(\Sigma^{(1)}_\infty\).\\
\underline{Step 2. The invariance of 
\(\Sigma^{(\ell+1)}_\infty\) based on the induction on \(\Sigma^{(\ell)}_\infty\).} Suppose that $\Sigma^{(\ell)}_\infty$ is invariant. We express \(\Sigma^{(\ell+1)}_\infty\) in \eqref{E:deep-Cov} as 
\(\Sigma^{(\ell+1)}_W + \gamma^2 \Sigma^{(\ell+1)}_b \) with 
\(\Sigma^{(\ell+1)}_b = \Sigma^{(1)}_b\) and 
\[ \Sigma^{(\ell+1)}_W(x^\AA, y^\BB)  = 
	\E_{f^{(\ell)}_\infty } \left[  \AA \sigma\left( f^{(\ell)}_\infty(x) \right) 
	\BB \sigma\left( f^{(\ell)}_\infty(y) \right) \right], \]
such that the random variables \(\AA \sigma( f^{(\ell)}_\infty(x))\) and 
\(\BB \sigma( f^{(\ell)}_\infty(y) )\) are given by \eqref{E:AA_sig(f_inf)}.
For the kernel \(\Sigma^{(\ell+1)}_W\), consider the image after the action by \(g\in \mathrm{GL}(d;\R)\): 
\begin{equation}\label{E:g-Sigma_W}
	\Sigma^{(\ell+1)}_W(g^{-1}\cdot x^\AA, g^{-1}\cdot y^\BB) 
	= \E_{f^{(\ell)}_\infty } \left[ (g^{-1}\cdot \AA) \sigma\left( f^{(\ell)}_\infty(g^{-1}x) \right) 
	(g^{-1}\cdot \BB) \sigma\left( f^{(\ell)}_\infty(g^{-1}y) \right) \right],
\end{equation}
with the random variable after group action given by 
\begin{equation}\label{E:g-AA_sig(f_inf)-imp} 
	(g^{-1}\cdot \AA) \sigma\left( f^{(\ell)}_\infty(g^{-1}x) \right)
	= \sum_{\alpha \in  I_{g^{-1}\cdot \AA} } a_\alpha (g^{-1}x) \sum_{\pi \in \Pi_\alpha} 
		\sigma^{(|\pi|)}\left( f^{(\ell)}_\infty \left( g^{-1}x \right) \right) 
		\prod_{A \in \pi}  f^{(\ell)}_\infty\left((g^{-1}x)^{\partial^A}\right).
\end{equation}
Instead of using the formula \eqref{E:g-AA_sig(f_inf)-imp}, we prove that for any 
\(g\in \mathrm{GL}(d;\R)\) and \(x^\AA\in\PLDR_d(\mathfrak{R})\), the random variables are given by
\begin{equation}\label{E:g-AA_sig(f_inf)}
	(g^{-1}\cdot \AA) \sigma\left( f^{(\ell)}_\infty(g^{-1}x) \right)
 	= \sum_{\alpha \in I_\AA} a_\alpha(x)
	\sum_{\pi \in \Pi_\alpha} \sigma^{(|\pi|)}\left(f^{(\ell)}_\infty(g^{-1}x)\right) 
	\prod_{A \in \pi} f^{(\ell)}_\infty\left((g^{-1}x)^{\wt{\partial}^A}\right).
\end{equation}
The proof relies heavily on the limit process of the infinite-width network output. 
Specifically, note that \(\AA \sigma( f^{(\ell)}_\infty(x))\) 
is the limit of \(\AA \sigma\left( f^{(\ell)}(x;\theta) \right)\), where \(f^{(\ell)}(x;\theta)\)
represents any component of the \(\ell\)th-layer outputs. By Lemma \ref{L:g-Faa di}, for any \(f:\R^d\to \R\), we have
\begin{align}\label{eqn:g-pi-0}
	(g^{-1}\cdot \AA) \sigma\left(f\right) 
	= \sum_{\alpha \in I_\AA} \wt{a}_\alpha(z) \wt{\partial}^\alpha\sigma\left(f\right) 
	= \sum_{\alpha \in I_\AA} \wt{a}_\alpha(z)
	  \sum_{\pi \in \Pi_\alpha} \sigma^{(|\pi|)}\left(f\right) 
	  \prod_{A \in \pi} \wt{\partial}^A f.
\end{align}
Therefore, equation \eqref{E:g-AA_sig(f_inf)} follows from the formula \eqref{eqn:g-pi-0} in the limit
\[ (g^{-1}\cdot \AA) \sigma\left( f^{(\ell)}_\infty(g^{-1}x) \right)
	= \lim_{d_1,\ldots,d_\ell \to \infty} (g^{-1}\cdot \AA) \sigma\left( f^{(\ell)}(x;\theta) \right).  \]
Note that \(g^{-1}\cdot \partial^A =\wt{\partial}^A\) and \(g^{-1}\cdot \EE = \EE\) with \(\EE\) being the identity operator. The identities \eqref{E:AA_sig(f_inf)} and \eqref{E:g-AA_sig(f_inf)} imply that 
their expectations and covariances are equal:
\begin{gather*}
	\E\left[ (g^{-1}\cdot \AA) \sigma\left( f^{(\ell)}_\infty(g^{-1}x) \right) \right]
 	= \E\left[ \AA \sigma\left( f^{(\ell)}_\infty(x) \right) \right], \\
	\Cov\left( (g^{-1}\cdot \AA) \sigma\left( f^{(\ell)}_\infty(g^{-1}x) \right), 
	(g^{-1}\cdot \BB) \sigma\left( f^{(\ell)}_\infty(g^{-1}y) \right) \right)
	= \Cov\left( \AA \sigma\left( f^{(\ell)}_\infty(x) \right), 
	\BB \sigma\left( f^{(\ell)}_\infty(y) \right)\right),
\end{gather*}
for any \(g\in\OO(d)\) and any \(x^\AA,y^\BB\in\PLDR_d(\mathfrak{R})\) by the invariance of \(\Sigma^{(\ell)}_\infty\). 
Thus we have \(\Sigma^{(\ell+1)}_W(g^{-1}\cdot x^\AA, g^{-1}\cdot y^\BB) 
= \Sigma^{(\ell+1)}_W(x^\AA, y^\BB) \) by equation \eqref{E:g-Sigma_W}. Thus,  \(\Sigma^{(\ell+1)}_\infty\) is invariant under \(\OO(d)\) and completes the induction step.\\
\underline{Step 3. The invariance of the DNTK \(\Theta^{(L)}_\infty\).} The proof is again based on mathematical induction.
In the base case, we have \(\Theta^{(1)}_\infty = \Sigma^{(1)}_\infty\) which has been proved.
Suppose the conclusion holds for some \(\ell\geq 2\). Note that the expression \eqref{E:deep-DNTK} of
\(\Theta^{(\ell+1)}_\infty\) can be split into two parts \(\Sigma^{(\ell+1)}_\infty\) and 
\(K^{(\ell+1)}_\infty\), corresponding to the outer and inner parameters respectively. The invariance 
of \(\Sigma^{(\ell+1)}_\infty\) has been proved. For the kernel \(K^{(\ell+1)}_\infty\), we have  
\[ K^{(\ell+1)}_\infty(x^\AA, y^\BB) = 
	\sum_{\alpha\in I_\AA }\sum_{\beta\in I_\BB } a_\alpha(x)b_\beta(y)
	\sum_{\xi\leq\alpha}\sum_{\zeta\leq\beta}\binom{\alpha}{\xi}\binom{\beta}{\zeta}
	\dot{\Sigma}^{(\ell+1)}_{\infty}\left( x^{\partial^\xi},y^{\partial^\zeta} \right)
	\Theta^{(\ell)}_{\infty}\left( x^{\partial^{(\alpha-\xi)}},y^{\partial^{(\beta-\zeta)}} \right). \]
Since computing 
\(K^{(\ell+1)}_\infty\left(g^{-1}\cdot x^\AA, g^{-1}\cdot y^\BB\right)\) directly is inconvenient, 
using the definition of the DNTK, we get 
\begin{align*}
	 K^{(\ell+1)}_\infty\left(g^{-1}\cdot x^\AA, g^{-1}\cdot y^\BB\right) 
	= \lim_{d_1,\ldots,d_\ell \to \infty}
	\llan \nabla_{\theta_{\leq\ell}}  (g^{-1}\cdot \AA)f^{(\ell+1)}(g^{-1}x;\theta), 
			\nabla_{\theta_{\leq\ell}}(g^{-1}\cdot \BB)f^{(\ell+1)}(g^{-1}y;\theta)  \rran,
\end{align*}
where \(f^{(\ell+1)}(x;\theta)\) denotes any component of the \(\ell\)th layer outputs 
and \(\theta_{\leq\ell}\) denotes the parameters lower than or equal to \(\ell\)-layer.
Using Lemma \ref{L:g-Leibniz}, we obtain for any \(g\in\mathrm{GL}(d;\R)\),
\begin{align*}
	K^{(\ell+1)}_\infty\left(g^{-1}\cdot x^\AA, g^{-1}\cdot y^\BB\right) 
	& = \sum_{\alpha\in I_\AA }\sum_{\beta\in I_\BB } a_\alpha(x)b_\beta(y)
	\sum_{\xi\leq\alpha}\sum_{\zeta\leq\beta}\binom{\alpha}{\xi}\binom{\beta}{\zeta} \cdot \\
	& \hspace{1cm}
	\dot{\Sigma}^{(\ell+1)}_{\infty}\left( (g^{-1}x)^{\wt{\partial}^\xi},(g^{-1}y)^{\wt{\partial}^\zeta} \right)
	\Theta^{(\ell)}_{\infty}\left( (g^{-1}x)^{\wt{\partial}^{(\alpha-\xi)}},(g^{-1}y)^{\wt{\partial}^{(\beta-\zeta)}} \right).
\end{align*}
Note that \(\dot{\Sigma}^{(\ell+1)}_{\infty}\) is the same as \(\Sigma^{(\ell+1)}_W\) 
with the activation function \(\sigma\) replaced by its derivative \(\sigma'\).
Thus the invariance of \(K^{(\ell+1)}_\infty\) under \(\OO(d)\) is obtained 
by the invariance of \(\dot{\Sigma}^{(\ell+1)}_{\infty}, \Sigma^{(\ell+1)}_{\infty}\) 
and the induction hypothesis on \(\Theta^{(\ell)}_{\infty}\).
\end{proof}

\begin{remark}
We expect that all the components \(\Theta^{(L)}_{W^{(\ell)}}\) and
\(\Theta^{(L)}_{b^{(\ell)}}\) for \(\ell\in[L]\), arising from the corresponding parameters, of the DNTK \(\Theta^{(L)}_{\infty}\) should be invariant under \(\OO(d)\) as in the SNNs case. In fact, \(\Theta^{(L)}_{W^{(L)}} = \Sigma^{(L)}_{W}\) and \(\Theta^{(L)}_{b^{(L)}} = \Sigma^{(L)}_{b}\) are symmetric, while the term \(K^{(L)}_\infty(x^\AA,y^\BB)\), a sum of all the components of \(\Theta^{(L)}_{\infty}\) corresponding to the inner layer parameters, is a summation of several expectation terms similar to \(\Sigma^{(L)}_{W}(x^\AA,y^\BB)\) \cite{JinWu:2026DNTK}. However, to prove the symmetry of the components of \(K^{(L)}_\infty(x^\AA,y^\BB)\), one needs much more effect to first derive their explicit expressions and then perform further computations like the SNNs case.
\end{remark}

Note that the proof of Theorem \ref{T:Sym-DNN} only relies on Lemmas \ref{L:g-Leibniz} and \ref{L:g-Faa di}, 
a variant of Leibniz rule and Faa di Bruno's formula under the differentiation system 
\(\wt{\partial} = g^{-1}\cdot \partial = g\partial\) for \(g\in\OO(d)\). 
Philosophically, if we transform an orthonormal basis of the Euclidean space to another and change 
the differentiation system accordingly, then the properties depend only on the derivatives of functions
will remain invariant. In fact, let \(\wt{z} = g^{-1}\cdot z = gz\) be the coordinate system and 
\(\wt{\partial}= g^{-1}\cdot \partial = g\partial \) the differentiation system of \(\R^d\)
under the transformation by some matrix \(g\in\OO(d)\) with rows \(g_1,\ldots,g_d\). 
Then for any \(i,j\in[d]\), we have 
\[ \wt{\partial}_i(\wt{z}_j) = (g^{-1}\cdot\partial_i)(g^{-1}\cdot z_j)
= (g_i \partial) (g_j z) = g_i g_j^\top = \delta_{ij} = \partial_i(z_j). \]
Thus, \(\wt{\partial} = g^{-1}\cdot \partial\) is the correct differentiation system in terms
of the new coordinate \(\wt{z} = g^{-1} \cdot z\) and motivates a more direct and natural proof below.

\begin{lemma}\label{L:system-trans} 
	Let \(\wt{z} = g^{-1}\cdot z = gz\) and \(\wt{\partial}= g^{-1}\cdot \partial = g\partial \) for some \(g\in\OO(d)\). 
	Then for any function \(f:\R^d\to \R\) and operator \(\AA\in \LD_d(\mathfrak{R})\) we have 
	\[  \wt{\AA}f(\wt{z}) \Big|_{\wt{z}=x} =  (g^{-1}\cdot \AA) (f\circ g) (z) \Big|_{z=g^{-1} x} =  \AA f(z) \Big|_{z=x}, \quad \forall  x\in\R^d. \]
\end{lemma}

\begin{proof}
Let \(\AA = \sum_{\alpha \in I_\AA} a_\alpha \partial^\alpha\). By the linearity of $\AA$, we need only prove the cases 
	\(\AA = a\) for some \(a\in\R\) or \(\AA=\partial^\alpha\) for some index \(\alpha\). 
In the first case, we have 
	\[  a(\wt{z}) f(\wt{z}) \Big|_{\wt{z}=x} =  a(gz) f(gz) \Big|_{z=g^{-1}x}
		 = a(x)f(x) =  a(z) f(z) \Big|_{z=x}.  \]
In the case \(\AA = \partial^\alpha\), first suppose \(|\alpha|=1\) so that \(\partial^\alpha = \partial_k\) 
	for some \(k\in[d]\). Then we have 
	\begin{align*}
		\wt{\partial}_k[f(\wt{z})] &  = \sum_{i=1}^{d} g_{ki}\partial_i[f(gz)]
		= \sum_{i=1}^{d} g_{ki} \sum_{j=1}^{d} (\partial_j f)(gz) \partial_i(g_jz)
		= \sum_{i=1}^{d} g_{ki} \sum_{j=1}^{d} (\partial_j f)(gz) g_{ji} \\
		& = \sum_{j=1}^{d} (\partial_j f)(gz) \sum_{i=1}^{d} g_{ki} g_{ji} 
		= (\partial_k f)(gz).
	\end{align*}
	Thus we arrive at 
	\[  \wt{\partial}_k[f(\wt{z})] \Big|_{\wt{z}=x} =  (\partial_k f)(gz) \Big|_{z=g^{-1}x}
		 = (\partial_k f)(x) =  \partial_k f(z) \Big|_{z=x}. \]
	For the case \(|\alpha|+1\), by induction, we have 
	\[ \wt{\partial}^\alpha\wt{\partial}_k[f(\wt{z})]\Big|_{\wt{z}=x} = 
		\wt{\partial}^\alpha [(\partial_k f)(\wt{z})] \Big|_{\wt{z}=x}
		= \partial^\alpha [(\partial_k f)(z)] \Big|_{z=x} 
		= \partial^\alpha\partial_k[f(z)]\Big|_{z=x}. \]	
\end{proof}

\begin{proof}[Second proof of Theorem \ref{T:Sym-DNN}]
Using Lemma \ref{L:system-trans}, we prove the symmetry of the covariance function \(\Sigma^{(L)}_{\infty}\) 
	and the DNTK \(\Theta^{(L)}_{\infty}\) directly using the recursive formulas. 
	For any \(g\in \OO(d)\), we have 
	\begin{align*}
		g\cdot \Sigma^{(\ell+1)}_\infty\left(x^\AA, y^\BB\right) 
		& = \Sigma^{(\ell+1)}_\infty\left((g^{-1}x)^{g^{-1}\cdot\AA}, (g^{-1}y)^{g^{-1}\cdot\BB}\right) \\
		& = \Cov\left( f_{\infty}^{(\ell+1)}\left((g^{-1}x)^{g^{-1}\cdot\AA}\right), 
			f_{\infty}^{(\ell+1)}\left((g^{-1}y)^{g^{-1}\cdot\BB}\right) \right).
	\end{align*}
The derivation of \(f_{\infty}^{(\ell+1)}\) in \cite{JinWu:2026DNTK}
and Lemma \ref{L:system-trans} imply
	\begin{align*}
		f_{\infty}^{(\ell+1)}\left((g^{-1}x)^{g^{-1}\cdot\AA}\right) 
		& = \lim_{d_1,\ldots,d_\ell \to \infty} (g^{-1}\cdot\AA) \sigma \left( f^{(\ell)}(g^{-1}x;\theta) \right) \\
		& = \lim_{d_1,\ldots,d_\ell \to \infty} (g^{-1}\cdot\AA) 
			\sigma \left( f^{(\ell)}(z;\theta) \right) \Big|_{z=g^{-1}x} \\
		& = \lim_{d_1,\ldots,d_\ell \to \infty} \wt{\AA} \sigma 
			\left( f^{(\ell)}\circ g^{-1}(\wt{z};\theta) \right)\Big|_{\wt{z}=x} \\
		& = \lim_{d_1,\ldots,d_\ell \to \infty} \AA \sigma 
			\left( \bar{f}^{(\ell+1)}(z;\theta) \right)\Big|_{z=x}
		  = \bar{f}^{(\ell+1)}_{\infty}(x^\AA), 
	\end{align*}
	with \(\bar{f}^{(\ell+1)} = g\cdot f^{(\ell+1)} = f^{(\ell+1)}\circ g^{-1}\) the original neural network 
	after the action by \(g\). Specifically, the NN \( g\cdot f^{(L)}(z;\theta)\) is given recursively
	for \(\ell=1,\ldots,L-1\) by 
	\begin{equation*}
	\begin{aligned}
		\bar{f}^{(1)}(z;\theta) & = W^{(1)}g^{-1} z  + \gamma b^{(1)},\\
		\bar{f}^{(\ell+1)}(x;\theta) & = \frac{1}{\sqrt{d_\ell}} W^{(\ell+1)} 
		\sigma \left( \bar{f}^{(\ell)}(x;\theta) \right) + \gamma b^{(\ell+1)}. 
	\end{aligned}
	\end{equation*}
	Therefore, we have \(g\cdot \Sigma^{(\ell+1)}_\infty\left(x^\AA, y^\BB\right) 
	= \ol{\Sigma}^{(\ell+1)}_\infty\left(x^\AA, y^\BB\right) \) such that \(\ol{\Sigma}^{(\ell+1)}_\infty\), 
	the covariance function of the infinite NN \(\bar{f}^{(\ell+1)}_\infty\), has the same recursive formula 
	as \(\Sigma^{(\ell+1)}_\infty\), due to the same recursion between \(f^{(L)}\) and \(\bar{f}^{(L)}\). Thus to prove \(g\cdot \Sigma^{(\ell+1)}_\infty 
	= \ol{\Sigma}^{(\ell+1)}_\infty = \Sigma^{(\ell+1)}_\infty\), it suffices to prove the base case.
	For any \(\mu\in[d]\), with \(h:=g^{-1}\in\OO(d)\), we have 
	\begin{align*}
		\bar{f}_\infty^{(1)}(x^\AA) & = \bar{f}_{\mu}^{(1)}\left(x^\AA;\theta\right) 
		  = \AA\left( W^{(1)}_{\mu} hx + \gamma b^{(1)}_{\mu} \right) \\
		& = \sum_{|\alpha|\leq 1} a_\alpha(x)\partial^\alpha\left( W^{(1)}_{\mu} hx + \gamma b^{(1)}_{\mu} \right) \\
		& = \sum_{i,k=1}^{d} a_i(x) W^{(1)}_{\mu,k}h_{ki} + a(x)\left( W^{(1)}_{\mu} hx + \gamma b^{(1)}_{\mu} \right) \\
		& =  W^{(1)}_{\mu} h\big(a(x)x+D\AA(x)\big) + \gamma a(x) b^{(1)}_{\mu},
	\end{align*}
	with \(a_i\) being the coefficient of \(\partial_i\) in \(\AA\) and \(a\) being the zeroth term of \(\AA\), 
	from which we obtain 
$$\ol{\Sigma}^{(1)}_\infty\left(x^\AA,y^\BB\right) 
= \Sigma^{(1)}_\infty\left(x^\AA,y^\BB\right) \quad \forall  x^\AA,y^\BB\in \PLDR_d(R),$$ cf. \eqref{E:deep-Cov}. 
In fact, it follows directly since the parameters 
\(W^{(1)}_{\mu} h\) is a standard Gaussian vector as \(W^{(1)}_{\mu}\).
	For the recursive formula of \(\Theta^{(L)}_{\infty}\), by Lemma \ref{L:system-trans}, we have for \(g\in\OO(d)\), 
	\begin{align*}
		g\cdot \Theta^{(\ell+1)}_\infty\left(x^\AA, y^\BB\right) 
		& = \lim_{d_1,\ldots,d_\ell \to \infty}
		\llan \nabla_{\theta}  (g^{-1}\cdot \AA)f^{(\ell+1)}(g^{-1}x;\theta), 
		\nabla_{\theta}(g^{-1}\cdot \BB)f^{(\ell+1)}(g^{-1}y;\theta)  \rran, \\
		& = \lim_{d_1,\ldots,d_\ell \to \infty}
		\llan \nabla_{\theta}  (g^{-1}\cdot \AA)f^{(\ell+1)}(z;\theta)|_{z=g^{-1}x}, 
		\nabla_{\theta}(g^{-1}\cdot \BB)f^{(\ell+1)}(z;\theta)_{z=g^{-1}y}  \rran, \\
		& = \lim_{d_1,\ldots,d_\ell \to \infty}
		\llan \nabla_{\theta}  \wt{\AA}(f^{(\ell+1)}\circ g^{-1})(\wt{z};\theta)|_{\wt{z}=x}, 
		\nabla_{\theta}\wt{\BB}(f^{(\ell+1)}\circ g^{-1})(\wt{z};\theta)|_{\wt{z}=y}  \rran \\
		& = \lim_{d_1,\ldots,d_\ell \to \infty}
		\llan \nabla_{\theta}  \AA\bar{f}^{(\ell+1)}(z;\theta)|_{z=x}, 
		\nabla_{\theta}\BB\bar{f}^{(\ell+1)}(z;\theta)|_{z=y}  \rran,
	\end{align*}
Therefore, through the limit process in \cite{JinWu:2026DNTK}, we  get 
	\( g\cdot \Theta^{(\ell+1)}_\infty\left(x^\AA, y^\BB\right) = 
 		 \ol{\Theta}^{(\ell+1)}_\infty\left(x^\AA, y^\BB\right) \) 
	such that the DNTK \(\ol{\Theta}^{\ell+1}_\infty\) of the NN \(\bar{f}^{(L)}\) 
	has the recursive formula \eqref{E:deep-DNTK} 
	of \(\Theta^{(\ell+1)}_\infty\), except that the terms \(\Sigma^{(\ell+1)}_{\infty},
	\dot{\Sigma}^{(\ell+1)}_{\infty}\) and \(\Theta^{(\ell)}_{\infty}\) therein replaced by 
	\(\ol{\Sigma}^{(\ell+1)}_{\infty}, \dot{\ol{\Sigma}}^{(\ell+1)}_{\infty}\) and 
	\(\ol{\Theta}^{(\ell)}_{\infty}\). 
	Thus once the base case \(\ol{\Theta}^{(1)}_{\mu\nu} = \Theta^{(1)}_{\mu\nu}\) for all \(\mu,\nu\in[d]\) is proved, 
	the proof is then completed by mathematical induction and the equality \(\ol{\Sigma}^{(\ell+1)}_{\infty} = 
	\Sigma^{(\ell+1)}_\infty\). Indeed, with \(h:=g^{-1}\in\OO(d)\), we have
	\begin{align*}
		&\ol{\Theta}^{(1)}_{\mu\nu} \left( x^\AA,y^\BB \right) 
		 = \llan \nabla_\theta \AA \left(W^{(1)}_\mu hx + \gamma b^{(1)}_\mu\right), 
			\nabla_\theta \BB \left(W^{(1)}_\nu hy + \gamma b^{(1)}_\nu\right) \rran \\
	 =& \nabla_\theta \left( W^{(1)}_{\mu} h\big(a(x)x+D\AA(x)\big) + \gamma a(x) b^{(1)}_{\mu} \right)\cdot 
			\nabla_\theta \left( W^{(1)}_{\nu} h\big(b(y)y+D\BB(y)\big) + \gamma b(y) b^{(1)}_{\nu} \right) \\
	 =& \delta_{\mu\nu} \left(a(x)x+D\AA(x)\right)^\top \left(b(y)y+D\BB(y)\right) 
			+ \delta_{\mu\nu} \gamma^2 a(x)b(y) 
		= \delta_{\mu\nu} \Theta^{(1)}_{\infty}\left( x^\AA,y^\BB \right).
	\end{align*}
\end{proof}

\section{Proofs in Section \ref{S:Conseq}} \label{S-A:Conseq}

By Theorem \ref{T:Sym-DNN}, we have for all \(x,y\in\R^d, \AA\in\LD_d(\mathfrak{R})\) and \(g\in \OO(d)\), 
\begin{equation} \label{E:symmetry}
	\Theta^{(L)}_{g\cdot\AA}(x,y) = \Theta^{(L)}_\infty(x^{g\cdot\AA}, y^{g\cdot\AA}) 
	= \Theta^{(L)}_\infty\left((g^{-1}x)^{\AA}, (g^{-1}y)^{\BB}\right) = \Theta^{(L)}_{\AA}(g^{-1}x,g^{-1}y).
\end{equation}

\subsection{Proof of Proposition \ref{P:AA=g-BB}}

Given any \(f\in L^2(\mathbb{S}^{d-1})\), by equation \eqref{E:symmetry} and the coordinate transformation 
\(y=gw\) for \(g\in \OO(d)\), we have for any \(x\in\mathbb{S}^{d-1}\),
\begin{align*}
	T^{(L)}_\BB [f](x) = \int_{\mathbb{S}^{d-1}} \Theta^{(L)}_{\AA}(g^{-1}x,g^{-1}y)f(y) ~ \d S^{d-1}(y)
	= \int_{\mathbb{S}^{d-1}} \Theta^{(L)}_{\AA}(g^{-1}x,w)f(gw) ~ \d S^{d-1}(w).
\end{align*}
For the term \(M_{g} T^{(L)}_\AA M_{g^{-1}}\), we have
\begin{align*}
	\Big(T^{(L)}_{\AA}M_{g^{-1}}\Big) [f](x) & = T^{(L)}_{\AA}[f\circ g](x) 
	= \int_{\mathbb{S}^{d-1}} \Theta^{(L)}_{\AA}(x,y)(f\circ g)(y) ~ \d S^{d-1}(y) \\
 &	= \int_{\mathbb{S}^{d-1}} \Theta^{(L)}_{\AA}(x,y)f(gy) ~ \d S^{d-1}(y), 
\end{align*}
from which we obtain 
\[ \Big( M_g T^{(L)}_{\AA}M_{g^{-1}}\Big) [f](x) 
= \int_{\mathbb{S}^{d-1}} \Theta^{(L)}_{\AA}(g^{-1}x,y)f(gy) ~ \d S^{d-1}(y) = T^{(L)}_\BB [f](x). \] 
If \(T^{(L)}_{\AA}f = \lambda f\) for some \(\lambda\) and \(f\), 
then by the transformation \(w=gy\), we have for \(x\in \mathbb{S}^{d-1}\), 
\begin{align*}
	T^{(L)}_\BB [g\cdot f](x) 
	& = \int_{\mathbb{S}^{d-1}} \Theta^{(L)}_{\BB}(x,y)(f\circ g^{-1})(y) ~ \d S^{d-1}(y) 
	  = \int_{\mathbb{S}^{d-1}} \Theta^{(L)}_{\AA}(g^{-1}x,g^{-1}y)f(g^{-1}y) ~ \d S^{d-1}(y) \\
	& = \int_{\mathbb{S}^{d-1}} \Theta^{(L)}_{\AA}(g^{-1}x,w)f(w) ~ \d S^{d-1}(w)
	  = \lambda f(g^{-1}x) = \lambda (g\cdot f)(x).
\end{align*}

\subsection{Proof of Proposition \ref{P:zonal}}

If \(\AA\in\LD_d(C(\mathbb{S}^{d-1}))^{\OO(d)}\), then \(g\cdot \AA=\AA\) in \eqref{E:symmetry}, 
which implies \(\Theta^{(L)}_{\AA}(g^{-1}x,g^{-1}y) = \Theta^{(L)}_{\AA}(x,y)\) for all \(g\in\OO(d)\) and thus also the invariance of \(\Theta^{(L)}_{\AA}\) holds. 
For any two pairs \((x,y)\) and \((x_1,y_1)\) from the unit sphere such that \(x_1^\top y_1 = x^\top y\), 
the angles between the pairs must be equal. Then there exists an orthogonal matrix
\(g\in\OO(d)\) such that \((x_1,y_1) = (gx,gy)\). By the invariance of \(\Theta^{(L)}_{\AA}\), we have
$ \Theta^{(L)}_{\AA}(x_1,y_1) = \Theta^{(L)}_{\AA}(gx,gy) = \Theta^{(L)}_{\AA}(x,y),$
which shows the zonal property.

\subsection{Proof of Corollary \ref{C:formula-sym}}
To prove part (i), we first derive the formulas of \(\Theta^{(2)}_\infty\) for \(\EE,\RR\) and \(\LL\) 
from the formula \eqref{E:2-DNTK}. For the \(\EE\)-NTK \(\Theta^{(2)}_\EE\), we have 
\begin{equation}\label{E:identity}
\begin{aligned}
	\Theta^{(2)}_\EE(x, y) =& \E_{u,v}\left[ \sigma\left(u^\top x + \gamma v\right)
	\sigma\left(u^\top y + \gamma v\right) \right] + \gamma^2  \\
	&  + (x^\top y + \gamma^2) ~ \E_{u,v}
	\left[ \sigma'\left(u^\top x + \gamma v\right)
	\sigma'\left(u^\top y + \gamma v\right)  \right].
\end{aligned}
\end{equation}
For the radial derivative operator \(\RR\), we have \(\RR_{(k)} = x_k \) by Lemma \ref{L:AA-z_k}. First, we have 
\[ \RR\sigma\left(u^\top x + \gamma v\right)
= \sum_{i=1}^{d} x_i \partial_{i}\sigma\left(u^\top x + \gamma v\right)
= \sum_{i=1}^{d} x_i u_i \sigma'\left(u^\top x + \gamma v\right) 
= u^\top x ~ \sigma'\left(u^\top x + \gamma v\right), \]
which implies
\begin{align*}
	\mathcal{R}\left(\sigma'\left(u^\top x + \gamma v\right)x_k\right) 
	& = x_k \RR \sigma'( u^\top x + \gamma v ) + \RR_{(k)}\sigma'( u^\top x + \gamma v ) \\
	& = x_k\cdot u^\top x ~ \sigma''(u^\top x + \gamma v) 
		+ x_k\cdot \sigma'( u^\top x + \gamma v ).
\end{align*}
Then, for the third term in \(\Theta^{(2)}_\RR\), we have
\begin{align*}
	& \sum_{k=1}^{d}\mathcal{R}\left(\sigma'\left(u^\top x + \gamma v\right)x_k\right)
	\mathcal{R}\left(\sigma'\left(u^\top y + \gamma v\right)y_k\right) \\
	 =& x^\top y ~ \left( u^\top x ~ \sigma''\left(u^\top x + \gamma v\right) 
	+ \sigma'( u^\top x + \gamma v ) \right) 
	\left( u^\top y ~ \sigma''\left(u^\top y + \gamma v\right) 
	+ \sigma'( u^\top y + \gamma v ) \right).
\end{align*}
By combining these relations, we obtain 
\begin{equation}\label{E:radial}
\begin{aligned}
	\Theta^{(2)}_\RR(x, y)
	= &\E_{u,v}\left[ u^\top x ~ \sigma'\left(u^\top x + \gamma v\right)
		u^\top y ~ \sigma'\left(u^\top y + \gamma v\right) \right] \\
	& + x^\top y ~ \E_{u,v}\left[ \left( u^\top x ~ \sigma''(u^\top x + \gamma v) 
		+ \sigma'( u^\top x + \gamma v ) \right) 
		\left( u^\top y ~ \sigma''(u^\top y + \gamma v) 
		+ \sigma'( u^\top y + \gamma v ) \right)  \right]\\
	&+ \gamma^2 \E_{u,v}\left[ u^\top x ~ \sigma''\left(u^\top x + \gamma v\right)
		u^\top y ~ \sigma''\left(u^\top y + \gamma v\right) \right].
\end{aligned}
\end{equation}
For the Laplace operator $\LL$, we have \(\LL_{(k)} = 2\partial_{k} \) by Lemma \ref{L:AA-z_k}. First, we have 
\[ \LL\sigma\left(u^\top x + \gamma v\right)
= \sum_{i=1}^{d}  \partial_{i}^2 \sigma\left(u^\top x + \gamma v\right)
= \sum_{i=1}^{d} u_i^2 \sigma''\left(u^\top x + \gamma v\right) 
= u^\top u ~ \sigma''\left(u^\top x + \gamma v\right), \]
which implies 
\begin{align*}
	\mathcal{L}\left(\sigma'\left(u^\top x + \gamma v\right)x_k\right) 
	& = x_k \LL \sigma'( u^\top x + \gamma v ) + \LL_{(k)}\sigma'( u^\top x + \gamma v ) \\
	& = x_k u^\top u ~ \sigma'''(u^\top x + \gamma v)
		+ 2 u_k \sigma''( u^\top x + \gamma v ).
\end{align*}
Then, each summand in the third term of \(\Theta^{(2)}_\LL\) is given by
\begin{align*}
	& \mathcal{L}\left(\sigma'\left(u^\top x + \gamma v\right)x_k\right) 
	  \mathcal{L}\left(\sigma'\left(u^\top y + \gamma v\right)x_k\right) \\
	=& \left( x_k u^\top u ~ \sigma'''(u^\top x + \gamma v)
		+ 2 u_k \sigma''( u^\top x + \gamma v ) \right)
		\left( y_k u^\top u ~ \sigma'''(u^\top y + \gamma v)
		+ 2 u_k \sigma''( u^\top y + \gamma v ) \right) \\
	=& x_k y_k (u^\top u)^2 \sigma'''(u^\top x + \gamma v) \sigma'''(u^\top y + \gamma v)
		+ 4 u_k^2 \sigma''( u^\top x + \gamma v )\sigma''( u^\top y + \gamma v ) \\
	& + 2x_k u_k ~u^\top u~ \sigma'''(u^\top x + \gamma v) \sigma''( u^\top y + \gamma v )
		+ 2y_k u_k ~u^\top u~ \sigma'''(u^\top y + \gamma v) \sigma''( u^\top x + \gamma v ),
\end{align*}
from which we obtain 
\begin{align*}
	& \sum_{k=1}^{d}\mathcal{L}\left(\sigma'\left(u^\top x + \gamma v\right)x_k\right)
	\mathcal{L}\left(\sigma'\left(u^\top y + \gamma v\right)y_k\right) \\
=& x^\top y ~ (u^\top u)^2 \sigma'''(u^\top x + \gamma v) \sigma'''(u^\top y + \gamma v)
		+ 4 u^\top u~ \sigma''( u^\top x + \gamma v )\sigma''( u^\top y + \gamma v ) \\
	& + 2u^\top x ~u^\top u~ \sigma'''(u^\top x + \gamma v) \sigma''( u^\top y + \gamma v )
		+ 2u^\top y ~u^\top u~ \sigma'''(u^\top y + \gamma v) \sigma''( u^\top x + \gamma v ).
\end{align*}
By combining the preceding identities, we obtain 
\begin{equation}\label{E:Laplace}
\begin{aligned}
	\Theta^{(2)}_\LL(x, y)
	=& \E_{u,v}\left[ (u^\top u)^2 \sigma''\left(u^\top x + \gamma v\right)
		 \sigma''\left(u^\top y + \gamma v\right) \right] \\
	&  + \E_{u,v}\Big[ x^\top y ~ (u^\top u)^2 \sigma'''(u^\top x + \gamma v) \sigma'''(u^\top y + \gamma v)
		+ 4 u^\top u~ \sigma''( u^\top x + \gamma v )\sigma''( u^\top y + \gamma v ) \\
	& +   2u^\top u ~u^\top x~ \sigma'''(u^\top x + \gamma v) \sigma''( u^\top y + \gamma v )
		+ 2u^\top u ~u^\top y~ \sigma'''(u^\top y + \gamma v) \sigma''( u^\top x + \gamma v )  \Big]\\
	& + \gamma^2 \E_{u,v}\left[ (u^\top u)^2 \sigma'''\left(u^\top x + \gamma v\right)
		 \sigma'''\left(u^\top y + \gamma v\right) \right].
\end{aligned}	
\end{equation}
When \(\sigma=\sigma_s\) and \(v=0\) in equation \eqref{E:identity}, we get 
\begin{align*}
	K^\EE_s(x, y) & = \E_{u}\left[ \sigma_s\left(u^\top x\right)\sigma_s\left(u^\top y\right) \right] 
		+ \gamma^2 + s^2(x^\top y + \gamma^2) ~ \E_{u,v}\left[ \sigma_{s-1}\left(u^\top x\right)
		\sigma_{s-1}\left(u^\top y\right)  \right] \\
	& = k_s(x,y) + \gamma^2 + s^2(x^\top y) k_{s-1}(x,y) + \gamma^2 s^2 k_{s-1}(x,y).
\end{align*}
Similarly, for the operator \(\RR\), we get 
$$K^\RR_s(x, y) = s^2 k_s(x,y) + s^4 (x^\top y) k_{s-1}(x,y) + \gamma^2 (s)_2^2 k_{s-1}(x,y).$$
While for the operator \(\LL\), by Lemma \ref{L:k_s(x,y;q)} below, we obtain
\begin{align*}
	K^\LL_s(x, y) & = (s)_2^2 k_{s-2}(x,y;2) + (s)_3^2(x^\top y) k_{s-3}(x,y;2) + 4(s)_2^2 k_{s-2}(x,y;1) \\
		& \qquad + 4(s)_2^2(s-2) k_{s-2}(x,y;1) + \gamma^2 (s)_3^2 k_{s-3}(x,y;2).
\end{align*}

\begin{lemma}\label{L:k_s(x,y;q)}
Let \(k_s(x,y)\) be the arc-cosine kernel \eqref{E:arc-cos} restricted on the unit sphere \(\mathbb{S}^{d-1}\). Given any integer \(p\geq 0\), we have for any \(x,y\in \mathbb{S}^{d-1}\), 
\begin{equation*}
	k_s(x,y;p) := \E_{u\sim\NN(0,I_d)}\left[ (u^\top u)^p \sigma_s(u^\top x) \sigma_s(u^\top y) \right]
		= \frac{m_{s+p}\left(\chi^2_{d}\right)}{m_{s}\left(\chi^2_{d}\right)} k_s(x,y). 
	\end{equation*}
\end{lemma}

\begin{proof}
	Note that the Gaussian vector \(u\sim \NN(0,I_d)\) can be decomposed as the product of two independent
	random variables \( u = |u|\cdot \ol{u} \), such that the norm square \(|u|^2 \sim \chi^2_d \) has 
	\(\chi^2_d\)-distribution, and the direction \( \ol{u} = u/|u| \sim \mathrm{U}(\mathbb{S}^{d-1})\) 
	is uniformly distributed over \(\mathbb{S}^{d-1}\). Thus, by definition, we have 
	\begin{align*}
		k_s(x,y;p) & = \E_{u\sim \NN(0,I_d)}\left[ (u^\top u)^p \sigma_s(u^\top x) \sigma_s(u^\top y) \right]\\
			& = \E_{u\sim \NN(0,I_d)}\left[ |u|^{2(p+s)} \sigma_s(\ol{u}^\top x) \sigma_s(\ol{u}^\top y) \right]\\
			& = \E_{|u|^2\sim \chi^2_d}\big[|u|^{2(p+s)}\big] \cdot
				\E_{\ol{u}\sim\mathrm{U}(\mathbb{S}^{d-1})}\left[ \sigma_s(\ol{u}^\top x) \sigma_s(\ol{u}^\top y) \right]\\
			& = m_{p+s}\left(\chi^2_{d}\right) 
				\E_{\ol{u}\sim\mathrm{U}(\mathbb{S}^{d-1})}\left[ \sigma_s(\ol{u}^\top x) \sigma_s(\ol{u}^\top y) \right].
	\end{align*}	
	Similarly, we have \(k_s(x,y)=m_{s}\left(\chi^2_{d}\right) 
	\E_{\ol{u}\sim\mathrm{U}(\mathbb{S}^{d-1})}\left[ \sigma_s(\ol{u}^\top x) \sigma_s(\ol{u}^\top y) \right]\).
This completes the proof of the lemma. 
\end{proof}

To prove part (ii), we need to compute the eigenvalue decay rate of the kernel \((x^\top y)k_s(x,y)\) 
with the knowledge of the arc-cosine kernel. To this end, we study  
\((x^\top y)^q k_s(x,y)\) for some integer \(q\geq0\) and first consider the reduction 
of the index \(q\) through the recursive formula of Lengendre polynomials. 
\begin{lemma}\label{L:reduction-q}
	Let \(P_{k} = P_{k,d}\) be the Lengendre polynomial of degree \(k\) in \(d\) dimension,
	and fix the integer \(q\in\N\). Then for any \(k>q\), we have 
	\[ t^q P_k(t) = c_{k-q}P_{k-q}(t) + c_{k-q+2}P_{k-q+2}(t) + \cdots + 
		c_{k-q+2(q-1)}P_{k-q+2(q-1)}(t) + c_{k+q}P_{k+q}(t), \]
	where for all \(i=0,1,\ldots,q\), the coefficients \(c_{k-q+2i} > 0\) and satisfy
	\( \lim_{k\to\infty} c_{k-q+2i} = 2^{-q}\binom{q}{i} \).
\end{lemma}

\begin{proof}

For \(d\geq 2\) the Lengendre polynomials satisfy \(P_0(t) = 1, P_1(t) =t\) and 
the recursive formula \cite[Eq. (2.86)]{Atkinson:2012}:
\begin{equation}\label{E:Legen-recur}
	\begin{cases}
		tP_k(t)  = P_{k+1}(t), &  k =0, \\
		tP_k(t)  = {\displaystyle\frac{k}{2k+d-2}}P_{k-1}(t) + 
				{\displaystyle\frac{k+d-2}{2k+d-2}}P_{k+1}(t), & k\geq 1.
	\end{cases}
\end{equation}
For the function \(t^q P_k(t)\), using the recursive formula \(q\) times, we get a positive 
linear combination of \(P_{k-q+2i}(t)\) for \(i=0,1,\ldots,q\) whose coefficients corresponds 
to the weighted downward paths from \((k,q)\) to the points with coordinates \(k-q+2i\) on the 
\(k\)-axis in Fig. \ref{F:recursion-k>q}. 
The weight of each path is the product of the weights of all its segments. 
For each segment \(a\) with the upper end \(U(a) = (x,y)\), its weight \(w_a\)
depends on its lower end \(L(a)\) and according to the formula \eqref{E:Legen-recur} is given by 
\begin{equation*}
	w_a = 
	\begin{cases}
		\displaystyle{\frac{x}{2x+d-2}}, & \text{if }   L(a) = (x-1,y-1), \\
		& \\
		\displaystyle{\frac{x+d-2}{2x+d-2}}, & \text{if }   L(a) = (x+1,y-1).
	\end{cases}
\end{equation*}
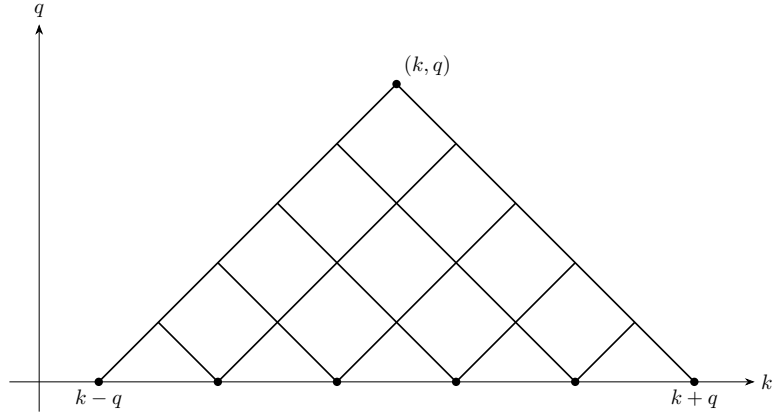
\begin{figure}[!ht]
\centering
\resizebox{0.6\textwidth}{!}{%
\begin{tikzpicture}
\draw [line width=0.2pt, ->, >=Stealth] (-0.5,0) -- (12,0);
\node at (12,0) [right] {\(k\)};
\draw [->, >=Stealth] (0,-0.5) -- (0, 6);
\node at (0,6) [above] {\(q\)};
\draw[thick] (1,0) to  (6,5);
\draw[thick] (3,0) to  (7,4);
\draw[thick] (5,0) to  (8,3);
\draw[thick] (7,0) to  (9,2);
\draw[thick] (9,0) to  (10,1);
\draw[thick] (6,5) to  (11,0);
\draw[thick] (5,4) to  (9,0);
\draw[thick] (4,3) to  (7,0);
\draw[thick] (3,2) to  (5,0);
\draw[thick] (2,1) to  (3,0);


\node at (5,0) [below] {};
\fill (5,0) circle[radius=2pt];
\node at (3,0) [below] {};
\fill (3,0) circle[radius=2pt];
\node at (1,0) [below] {\(k-q\)};
\fill (1,0) circle[radius=2pt];

\node at (7,0) [below] {};
\fill (7,0) circle[radius=2pt];
\node at (9,0) [below] {};
\fill (9,0) circle[radius=2pt];
\node at (11,0) [below] {\(k+q\)};
\fill (11,0) circle[radius=2pt];

\node at (6,5) [above right] {\((k,q)\)};
\fill (6,5) circle[radius=2pt];
\end{tikzpicture}
}
\caption{The recursion of Lengendre Polynomials for \(k>q\)}
\label{F:recursion-k>q}
\end{figure}
Let \(\mathcal{P}(k-q+2i)\) denote the set of all paths from \((k,q)\) to \((k-q+2i,0)\). 
Then the coefficients \(c_{k-q+2i}\) are 
\begin{align*}
	c_{k-q+2i} = \sum_{\ell\in\mathcal{P}(k-q+2i)} w_\ell 
	= \sum_{\ell\in\mathcal{P}(k-q+2i)} \prod_{a\in \ell} w_a.
\end{align*}
Note that all the downward paths have \(q\) segments and there are totally \(\binom{q}{i}\) paths 
in each \(\mathcal{P}(k-q+2i)\). When \(k\) tends to infinity, the weights of all the
segments \(w_a\) tend to \(1/2\) by the above formula. Therefore we have 
\[ \lim_{k\to\infty} c_{k-q+2i} = \sum_{\ell\in\mathcal{P}(k-q+2i)} \prod_{a\in \ell} 
	\lim_{k\to\infty} w_a = \binom{q}{i}\frac{1}{2^{q}}. \]

\end{proof}

Next, we compute the eigenvalues of the kernel \((x^\top y)^q k_s(x,y) \) which equals 
\(\frac{1}{2\pi}t^q J_s(t)\) by the formula \eqref{E:arc-cos} for \(t=x^\top y\) with \(x,y\in\mathbb{S}^{d-1}\)
and \(J_s(t) := J_s(\arccos t)\). In the case \(q=0\), from \cite[Appendix D]{Bach:2017breaking} and
\cite[Lemma 17]{Bietti:2019Inductive}, we have 
\begin{itemize}
	\item \(\lambda_k(J_s(t))>0\) for \(0\leq k\leq s\);
	\item \(\lambda_k(J_s(t)) = 0\) for \(k>s\) and \(k\equiv s \!\mod 2\);
	\item \(\lambda_k(J_s(t))>0\) and \(\lambda_k(J_s(t)) \sim k^{-d-2s}\) 
			for \(k> s\) and \(k\not\equiv s \!\mod 2\) when \(k\to \infty\).
\end{itemize} 
Next we show the kernel \((x^\top y)^q k_s(x,y)\) obeys a similar law as the arc-cosine 
kernel \(k_s(x,y)\).
\begin{lemma}\label{L:t^qJ_s(t)}
	Let \(J_s(t) := J_s(\arccos t)\) for \(t\in [-1,1]\) and \(q\in\N\). 
	Then the eigenvalue \(\lambda_k(t^q J_s(t))\) given by equation \eqref{E:F-H_coeff} satisfy
	\begin{itemize}
		\item \(\lambda_k(t^q J_s(t))>0\) for \(0\leq k\leq s+q\) ;
		\item \(\lambda_k(t^q J_s(t)) = 0\) for \(k>s+q\) and \(k\equiv s+q \!\mod 2\);
		\item \(\lambda_k(t^q J_s(t)) > 0\) and \(\lambda_k(t^q J_s(t)) \sim k^{-d-2s}\) 
		for \(k> s+q\) and \(k\not\equiv s+q \!\mod 2\) when \(k\to \infty\).
	\end{itemize} 
\end{lemma}

\begin{proof}

For \(0\leq k \leq q\), we use the recursive formula \eqref{E:Legen-recur} 
to represent \(t^qP_k(t)\) as a positive linear sum of \(P_{k+q-2i}(t)\) for  
\(i=0,1,\ldots, 2\lfloor \frac{k+q}{2} \rfloor\), cf. Fig. \ref{F:recursion-k<q}, where the weight 
of the segments with an upper end on the \(q\)-axis is \(1\) by \eqref{E:Legen-recur}.
\begin{figure}[!ht]
\centering
\resizebox{0.45\textwidth}{!}{%
\begin{tikzpicture}
\draw [line width=0.2pt, ->, >=Stealth] (-0.5,0) -- (8,0);
\node at (8,0) [right] {\(k\)};
\draw [->, >=Stealth] (0,-0.5) -- (0, 7);
\node at (0,7) [above] {\(q\)};
\draw[thick] (0,5) to  (1,6);
\draw[thick] (0,3) to  (2,5);
\draw[thick] (0,1) to  (3,4);
\draw[thick] (1,0) to  (4,3);
\draw[thick] (3,0) to  (5,2);
\draw[thick] (5,0) to  (6,1);
\draw[thick] (7,0) to  (1,6);
\draw[thick] (5,0) to  (0,5);
\draw[thick] (3,0) to  (0,3);
\draw[thick] (1,0) to  (0,1);

\node at (1,0) [below] {};
\fill (1,0) circle[radius=2pt];
\node at (3,0) [below] {};
\fill (3,0) circle[radius=2pt];
\node at (5,0) [below] {};
\fill (5,0) circle[radius=2pt];
\node at (7,0) [below] {\(k+q\)};
\fill (7,0) circle[radius=2pt];

\node at (1,6) [above right] {\((k,q)\)};
\fill (1,6) circle[radius=2pt];
\end{tikzpicture}
}
\caption{Recursion of Lengendre Polynomials for \(k\leq q\)}
\label{F:recursion-k<q}
\end{figure}
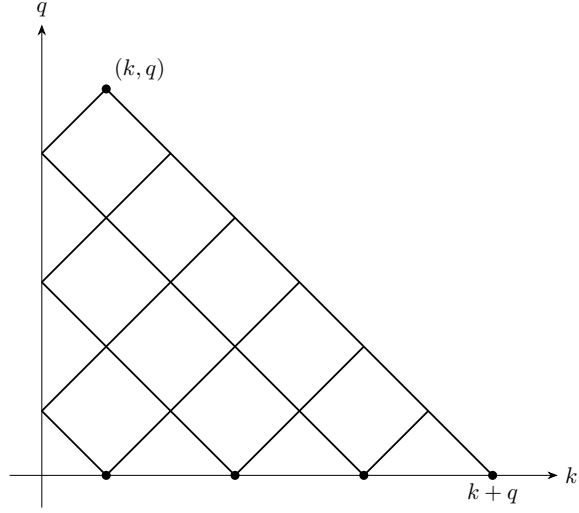
Therefore, \(\lambda_k(t^q J_s(t))\) is a positive linear sum of \(\lambda_{k+q-2i}(J_s(t))\)
for \(i=0,1,\ldots, 2\lfloor \frac{k+q}{2} \rfloor\) with the lowest index term given by 
\begin{equation*}
	\begin{cases}
		\lambda_0(J_s(t)) = \int_{-1}^{1} J_s(t) (1-t^2)^{\frac{d-3}{2}} ~ \d t, & \text{if } k\equiv q \!\mod 2,\\
		\lambda_1(J_s(t)) = \int_{-1}^{1} t J_s(t) (1-t^2)^{\frac{d-3}{2}} ~ \d t, & \text{if } k\not\equiv q \!\mod 2.
	\end{cases}
\end{equation*}
Note that \(\lambda_0(J_0(t))>0\) and both \(\lambda_0(J_s(t))\) and \(\lambda_1(J_s(t))\) 
are greater than zero for \(s\geq 1\). For the term \(\lambda_1(J_0(t))\), 
using the fact \(J_0(t) = \pi -\arccos t\) \cite{Cho:2009Kernel} we have
\begin{align*}
	\lambda_1(J_0(t)) & = \int_{-1}^{1} t (\pi -\arccos t) (1-t^2)^{\frac{d-3}{2}} ~ \d t \\
	& = \int_{0}^{1} t (\pi -\arccos t) (1-t^2)^{\frac{d-3}{2}} ~ \d t + \int_{-1}^{0} t (\pi -\arccos t) (1-t^2)^{\frac{d-3}{2}} ~ \d t \\
	& = \int_{0}^{1} t (\pi -2\arccos t) (1-t^2)^{\frac{d-3}{2}} ~ \d t > 0.
\end{align*}
For \(q< k \leq s+q\), the recursive formula \(\lambda_k(t^q J_s(t))\) is a 
positive linear sum of \(\lambda_{k-q+2i}(J_s), i=0,1,\ldots,q\) by Lemma \ref{L:reduction-q} with the first term 
\(\lambda_{k-q}(J_s)>0\) since \(k-q\leq s\).  
For the case \(k>s+q\), again we have 
\begin{align*}
	\lambda_k(t^q J_s(t)) = \sum_{i=0}^{q} c_{k-q+2i} \lambda_{k-q+2i}(J_s(t)). 
\end{align*}
If \(k\equiv s+q \!\mod 2\), then all the \(k-q+2i>s\) and \(k-q+2i \equiv s \!\mod 2 \) for 
\(i=0,1,\ldots,q\), thus \(\lambda_{k-q+2i}(J_s(t)) = 0\); 
if \(k\not\equiv q \!\mod 2\), then \(k-q+2i \not\equiv s \!\mod 2 \), thus
each 
$$\lambda_{k-q+2i}(J_s(t))\sim (k-q+2i)^{-d-2s}\sim k^{-d-2s}\quad \mbox{as }k\to\infty.$$ 
Since \(\lim_{k\to\infty} c_{k-q+2i} = 2^{-q}\binom{q}{i} \) by Lemma \ref{L:reduction-q}, we have  
\[ \lambda_k(t^q J_s(t)) = \sum_{i=0}^{q} c_{k-q+2i} \lambda_{k-q+2i}(J_s(t))
	\sim \sum_{i=0}^{q} c_{k-q+2i}k^{-d-2s} = k^{-d-2s} \quad \text{ as } \quad k\to\infty. \]

\end{proof}

Part (ii) of Corollary \ref{C:formula-sym} follows from Lemma \ref{L:t^qJ_s(t)} and we 
give the proof for the identity operator \(K^\EE_s\). 
For \(\gamma>0\), the constant kernel \(\gamma^2\) in \(K^\EE_s\) satisfies \(\lambda_0(\gamma^2)>0\) 
and \(\lambda_k(\gamma^2)=0\) for \(k>0\). By part (i), the kernel 
\(K^\EE_s(t) = k_s(t) + \gamma^2 + s^2 ~ t k_{s-1}(t) + \gamma^2 s^2 k_{s-1}(t)\) is a summation of positive kernels.
For \(0\leq k\leq s\), \(\lambda_k(K^\EE_s)>0\) since \(\lambda_k(k_s)>0\). 
For \(k>s\) and \(k\equiv s \!\mod 2\), we have \(k>s-1\) and \(k\not\equiv s-1 \!\mod 2\). 
Thus, \(\lambda_k(k_{s-1})>0\) and \(\lambda_k(k_{s-1}) \sim k^{-d-2(s-1)}\) while 
\(\lambda_k(k_{s}) = \lambda_k(tk_{s-1}) = 0\) as \(k\to\infty\).
For \(k>s\) and \(k\not\equiv s \!\mod 2\), we have \(\lambda_k(k_{s-1})=0\) while 
\(\lambda_k(k_{s})>0\) such that \(\lambda_k(k_{s}) \sim k^{-d-2s}\) and 
\(\lambda_k(tk_{s-1})>0\) such that \(\lambda_k(tk_{s-1}) \sim k^{-d-2(s-1)}\) as \(k\to \infty\).
By combining these relations together, we obtain the conclusion. Part (iii) follows from part (ii)  using the 
description \eqref{E:RKHS} of the RKHS since the spherical harmonics form an orthonormal basis 
in the Mercer decomposition of \(K^\EE_s,K^\RR_s\) and \(K^\LL_s\).

\subsection{Proof of Corollary \ref{C:sum-partial}}

For each \(i\in[d]\) and \(k\in[d]\), we have 
\[ \partial_i(\sigma'(u^\top x + \gamma v)x_k) 
= (x_k\partial_i + \delta_{ik})\sigma'(u^\top x + \gamma v)
= x_k u_i\sigma''(u^\top x + \gamma v) + \delta_{ik}\sigma'(u^\top x + \gamma v), \]
which implies 
\begin{align*}
	& \partial_i(\sigma'(u^\top x + \gamma v)x_k) \partial_i(\sigma'(u^\top y + \gamma v)y_k) \\
	=& x_k y_k u_i^2\sigma''(u^\top x + \gamma v)\sigma''(u^\top y + \gamma v) 
		+ \delta_{ik}^2 \sigma'(u^\top x + \gamma v)\sigma'(u^\top y + \gamma v) \\
	& + \delta_{ik}x_k u_i\sigma''(u^\top x + \gamma v)\sigma'(u^\top y + \gamma v)
		+ \delta_{ik}y_k u_i\sigma''(u^\top y + \gamma v)\sigma'(u^\top x + \gamma v).
\end{align*}
Then, for the third term in \(\Theta^{(2)}_\RR\), we have 
\begin{align*}
	& \sum_{k=1}^{d} \partial_i(\sigma'(u^\top x + \gamma v)x_k) \partial_i(\sigma'(u^\top y + \gamma v)y_k) \\
	=& x^\top yu_i^2\sigma''(u^\top x + \gamma v)\sigma''(u^\top y + \gamma v)  
		+ \sigma'(u^\top x + \gamma v)\sigma'(u^\top y + \gamma v) \\
	& + x_i u_i\sigma''(u^\top x + \gamma v)\sigma'(u^\top y + \gamma v)
		+ y_i u_i\sigma''(u^\top y + \gamma v)\sigma'(u^\top x + \gamma v).
\end{align*}
By combining with the other terms, we obtain
\begin{align*}
\Theta^{(2)}_{\partial_i}(x,y)  =& \E_{u,v}\left[u_i^2\sigma'(u^\top x + \gamma v)\sigma'(u^\top y + \gamma v)\right] \\
	& + \E_{u,v}\left[ x^\top y u_i^2\sigma''(u^\top x + \gamma v)\sigma''(u^\top y + \gamma v)  
		+ \sigma'(u^\top x + \gamma v)\sigma'(u^\top y + \gamma v) \right. \\
	&+ \left. x_i u_i\sigma''(u^\top x + \gamma v)\sigma'(u^\top y + \gamma v)
		+ y_i u_i\sigma''(u^\top y + \gamma v)\sigma'(u^\top x + \gamma v)	\right] \\
	& + \gamma^2 \E_{u,v}\left[u_i^2\sigma''(u^\top x + \gamma v)\sigma''(u^\top y + \gamma v)\right], 
\end{align*}
from which we obtain that the sum \(\Theta^{(2)}_{\partial}:=\Theta^{(2)}_{\partial_1}+\cdots+\Theta^{(2)}_{\partial_d}\)
is given by
\begin{align*}
	\Theta^{(2)}_{\partial}(x,y) =&  \E_{u,v}\left[u^\top u\sigma'(u^\top x + \gamma v)\sigma'(u^\top y + \gamma v)\right] \\
	&  + \E_{u,v}\left[ x^\top y u^\top u\sigma''(u^\top x + \gamma v)\sigma''(u^\top y + \gamma v)  
		+ \sigma'(u^\top x + \gamma v)\sigma'(u^\top y + \gamma v) \right. \\
	& \quad + \left. u^\top x\sigma''(u^\top x + \gamma v)\sigma'(u^\top y + \gamma v)
		+ u^\top y\sigma''(u^\top y + \gamma v)\sigma'(u^\top x + \gamma v)	\right] \\
	& + \gamma^2 \E_{u,v}\left[u^\top u\sigma''(u^\top x + \gamma v)\sigma''(u^\top y + \gamma v)\right].
\end{align*}
With \(\sigma=\sigma_s\) and \(v=0\), by Lemma \ref{L:k_s(x,y;q)}, we have 
\begin{align*}
	K^{\partial}_s(x,y) & = s^2 k_{s-1}(x,y;1) + (s)_2^2 (x^\top y)k_{s-2}(x,y;1) + s^2 k_{s-1}(x,y) 
		+ 2s^2(s-1)k_{s-1}(x,y) + \gamma^2 (s)_2^2 k_{s-2}(x,y;1) \\
	& = \left(\frac{m_{d,s}}{m_{d,s-1}}+2s-1\right) s^2 k_{s-1}(x,y) 
		+ \frac{m_{d,s-1}}{m_{d,s-2}} (s)_2^2 (x^\top y)k_{s-2}(x,y)
		+ \gamma^2 \frac{m_{d,s-1}}{m_{d,s-2}} (s)_2^2 k_{s-2}(x,y).
\end{align*}
The eigenvalue and the  RKHS of \(K^{\partial}_s\) are the same as Corollary \ref{C:formula-sym}. 
To prove the last equality, we first have \(\HH_{K^{\partial}_s}=H_{K^{\partial_1}_s}+\cdots+H_{K^{\partial_d}_s}\) 
from \cite[Section 6]{Aronszajn:1950RK}. By Proposition \ref{P:AA=g-BB} and the examples that follow, we know that \(H_{K^{\partial_j}_s}=P_{kj}\cdot H_{K^{\partial_k}_s}\) for any \(k,j\in[d]\), which finishes the proof.

\end{document}